%% file: iclr2027_conference.tex
\documentclass{article} 
\usepackage{iclr2027_conference,times}

\input{math_commands.tex}

\usepackage[
    colorlinks=true,
    citecolor=blue,
    linkcolor=blue,
    urlcolor=black
]{hyperref}
\usepackage{hyperref}
\usepackage{wrapfig}
\usepackage{url}
\usepackage{graphicx} 
\usepackage{cleveref}
\usepackage{amsthm}
\usepackage{algorithm}
\usepackage{listings}
\usepackage{xcolor}
\usepackage{graphicx}
\usepackage{subcaption}
\lstdefinestyle{favalg}{
    language=Python,
    basicstyle=\ttfamily\footnotesize,
    keywordstyle=\color{blue!60!black},
    commentstyle=\color{gray},
    stringstyle=\color{green!40!black},
    columns=fullflexible,
    keepspaces=true,
    showstringspaces=false,
    breaklines=true,
    frame=none,
    aboveskip=4pt,
    belowskip=4pt
}
\usepackage{xcolor}
\definecolor{cPA}{HTML}{2A78D6}    
\definecolor{cAES}{HTML}{8CB9F0}   
\usepackage[table]{xcolor}
\definecolor{festblue}{HTML}{D8E6F7}

\newtheorem{proposition}{Proposition}
\usepackage{booktabs}
\usepackage{caption}
\usepackage{subcaption}
\usepackage[table]{xcolor}
\definecolor{rowgray}{gray}{0.92}
\title{FestDPO: Few-step Generator Alignment \\ with Direct Preference Optimization}
\iclrfinalcopy
\author{%
  \href{mailto:jaewoo@kaist.ac.kr}{Jaewoo Lee}\textsuperscript{*,\,1,\,2}
  \And \href{mailto:kyuil.sim@kaist.ac.kr}{Kyuil Sim}\textsuperscript{*,\,1}
  \And Hyeongyu Kang\textsuperscript{1}
  \AND
  Kanghoon Lee\textsuperscript{1}
  \And Woocheol Shin\textsuperscript{1}
  \And Jinkyoo Park\textsuperscript{\textdagger,\,1,\,3}
  \AND
  \textnormal{%
  \textsuperscript{1}\,KAIST
  \quad
  \textsuperscript{2}\,MongooseAI
  \quad
  \textsuperscript{3}\,Omelet
  }
}

\begin{document}

\maketitle

\begin{abstract}

Few-step generative models can generate high-fidelity samples within a few function evaluations. Despite this efficiency, generated samples may not exhibit desirable properties. When these properties are difficult to encode as an explicit reward function, direct preference optimization (DPO) can align generative models using pairwise preference feedback without training a separate reward model. However, extending DPO to few-step generative models is challenging because few-step generative models are generally implicit, making the likelihood evaluation required by DPO intractable. To address this challenge, we introduce \textbf{Few-step DPO (FestDPO)}, an extension of DPO for few-step generative models that leverages nonparametric likelihood estimation from empirical samples. By exploiting the fast sampling capabilities of few-step generative models, our approach makes sample-based approximation of DPO loss computationally feasible. Furthermore, the sample-based formulation makes FestDPO agnostic to the model family and sampling procedure. Our toy experiment demonstrates that FestDPO matches the reward-tilted target distribution across four few-step generators. For real-world tasks, we evaluate FestDPO in two domains: text-to-image generation and protein backbone generation. In text-to-image generation, FestDPO outperforms preference optimization baselines in both win rates against the base models and human evaluation scores. In protein backbone generation, it achieves a higher $\beta$-sheet fraction and better structural designability than the baselines.

\end{abstract}

\section{Introduction}

Recently emerging few-step generative models \citep{song2023consistency, frans2025one, zhou2025inductive,geng2026mean, deng2026generative} enable fast generation with high-quality samples. However, generated samples may not satisfy user preferences \citep{xu2023imagereward, wu2023human}, including undesirable content \citep{park2024direct, liu2025alignguard} or potentially harmful biological molecules \citep{das2021accelerated}. Encoding these desiderata in an explicit reward function can be difficult, whereas pairwise comparisons provide a way to express preferences for alignment \citep{christiano2017deep,wirth2017survey}. 
Direct Preference Optimization \citep[DPO;][]{rafailov2023direct} provides a solution for aligning generative models with preference pairs, operating without explicit reward models, and avoiding a potentially unstable online reinforcement learning (RL) loop. However, extending DPO to few-step generative models is challenging since the implicit nature of few-step models makes the likelihood evaluation \citep{ai2026joint} required by DPO intractable.

To overcome this limitation, we exploit the inherent sampling efficiency of few-step generative models. We found that the fast sampling speed makes sample-based nonparametric density estimation \citep{silverman2018density} computationally feasible. 
Building on this, we propose \textbf{Few-step DPO (FestDPO)}, a direct preference optimization method tailored to few-step generative models. DPO aims to sample from a reward-tilted unnormalized density, where the reward is derived from the Bradley-Terry preference model \citep{bradley1952rank}. To approximate DPO with few-step generative models, we replace its intractable likelihood terms with sample-based density estimation. The sample-based approach offers a key advantage: FestDPO is agnostic to the sampling procedure and pre-trained model family, allowing it to be applied to a broad range of few-step generative models without model-specific modifications. 
For the domains where distance in the original data space fails to reflect underlying structure, such as semantic similarity in images or rotational invariance in protein structures, we estimate densities in the embedding space of a pretrained encoder.

We evaluate FestDPO in a 1-D toy setting, text-to-image generation \citep{rombach2022high, esser2024scaling}, and protein backbone generation \citep{woo2026riemannian}, demonstrating effective preference optimization across tasks. In the toy setting, we show that FestDPO learns to sample from the unnormalized target density induced by the latent reward in the Bradley-Terry model across diverse few-step generative models.
In text-to-image generation at $512 \times 512$ resolution, FestDPO achieves the highest win rates among preference optimization baselines when fine-tuning both SDXL-Turbo \citep{sauer2024adversarial} and SDXL-DMD2 \citep{yin2024improved}, and receives the highest human ratings. Our ablation studies further show that FestDPO is robust to the choice of feature encoder and to the number of samples used for density estimation, indicating that sample-based estimation with a finite number of samples is effective for preference optimization, even in high-dimensional domains such as images and proteins. 
In protein backbone generation, we demonstrate that FestDPO on a few-step protein backbone generator \citep{woo2026riemannian} outperforms the baselines in structural designability and control over secondary structure composition of the protein. 


\vspace{-15pt}
\begin{figure}
    \centering
    \includegraphics[
        width=\linewidth,
        height=0.75\textheight,
        keepaspectratio
    ]{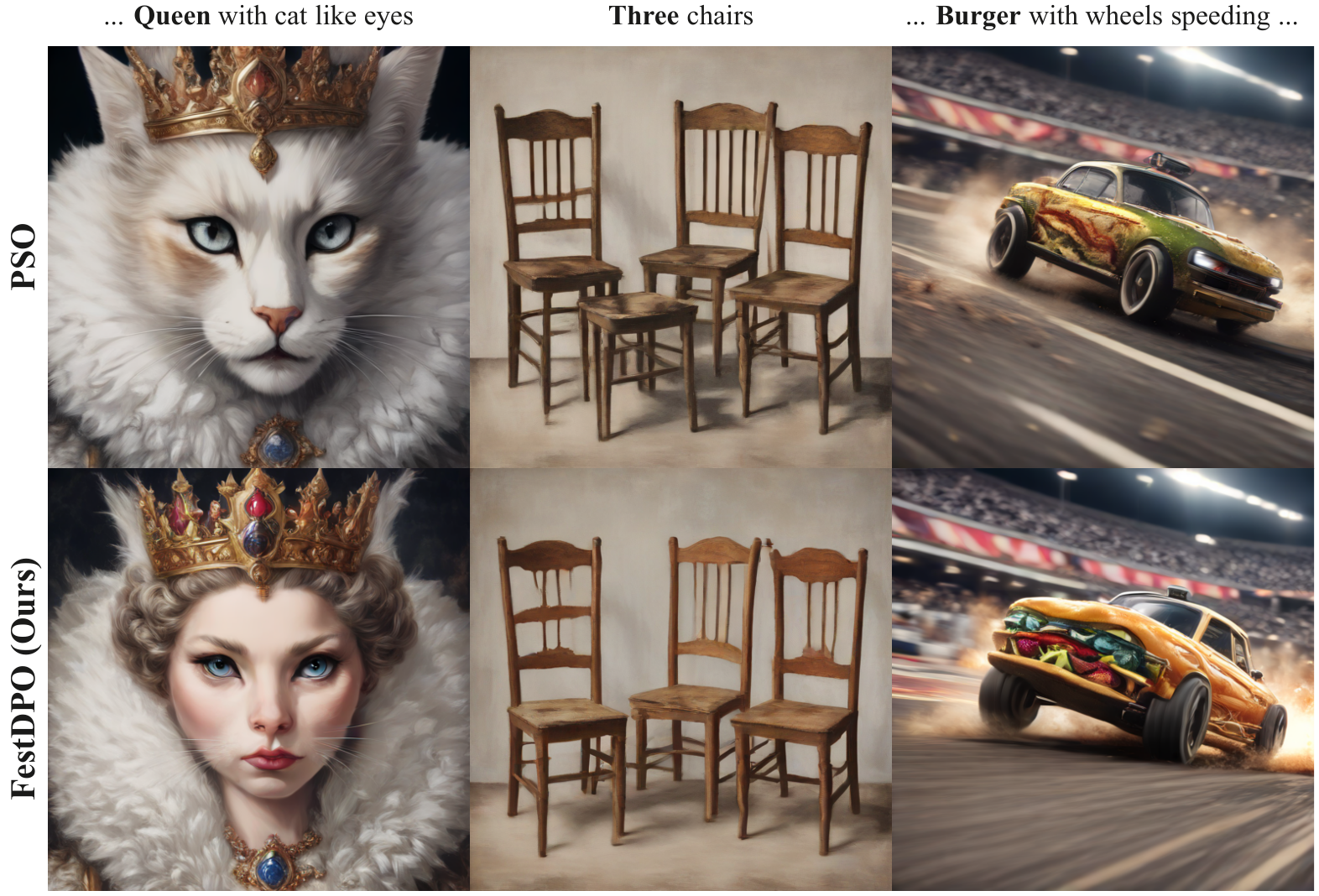}
    \caption{Qualitative comparison of PSO \citep{miao2025tuning} and FestDPO on three text prompts:\\ ``a gorgeous \textbf{queen} with cat-like eyes'' (Left), ``\textbf{three} chairs'' (Middle), and ``3D digital illustration, \textbf{burger with wheels} speeding on the race track, supercharged, detailed, hyperrealistic, 4K'' (Right).}
    \vspace{-15pt}
    \label{fig:main_qualitative}
\end{figure}

\section{Related Work}
\subsection{Few-step generative models}

Few-step generative models enable efficient sampling within a few function evaluations, addressing the expensive computational costs of diffusion \citep{ho2020denoising, song2020score} and flow-based models \citep{liu2022flow, lipman2022flow}. These models broadly fall into two categories. The first comprises direct noise-to-data mappings, such as VAEs \citep{kingma2013auto}, GANs \citep{goodfellow2020generative}, and recently proposed drifting models \citep{deng2026generative}. The second category is flow map models \citep{song2023consistency,kim2024consistency,boffi2024flow, frans2025one,geng2026mean,zhou2025inductive,boffi2026build}, including distilled generators \citep{luo2023diff,sauer2024adversarial,yin2024one}, which learn jump operators between intermediate states between noise and data. Despite advances in sampling speed, many of these models are implicit and do not provide tractable likelihoods \citep{ai2026joint} required by RL and probabilistic inference-based alignment methods \citep{uehara2024understanding, uehara2025inference}. Noise-to-data generators also do not exhibit denoising trajectories, which are required for score matching-based approaches \citep{honavar2025dspo}.

\subsection{Alignment}

Alignment is a core technique for steering generative models towards desirable output, broadly applied for natural language \citep{ziegler2019fine,stiennon2020learning, ouyang2022training}, images \citep{black2024training, fan2023dpok, clark2024directly, domingo2025adjoint, kang2026diffusion, lee2026reward, lee2026pcpo}, robotic control \citep{ren2025diffusion}, and scientific design \citep{gu2024aligning, lee2026diffusion, su2026iterative}. Many alignment methods rely on an explicit reward function, either predefined or learned from a labeled dataset. DPO \citep{rafailov2023direct} instead optimizes generators directly from offline preference pairs under a Bradley-Terry preference model \citep{bradley1952rank}, avoiding both an explicit reward model and an online RL loop. However, extending DPO-based alignment methods to emerging few-step generators is not straightforward: prior works often rely on diffusion-specific formulations \citep{wallace2024diffusion,yang2024using,croitoru2025curriculum,liu2026improving} or require likelihood evaluation on a denoising trajectory \citep{yang2024dense, liang2025aesthetic, honavar2025dspo, miao2025tuning}. Our concurrent work, DrPO \citep{jiang2026drifting}, uses preference pairs to construct drift fields \citep{deng2026generative} for aligning few-step generators, but its connection to the DPO objective remains unestablished. We instead develop a sample-based approximation to DPO and establish its asymptotic consistency, clarifying the corresponding reward-tilted target distribution.

\subsection{Sampling from unnormalized density}
Sampling from unnormalized density has been a fundamental problem of machine learning \citep{hinton2002training, lecun2006tutorial}. Given access to reward or energy evaluations, sampling methods can approximate the target distribution through Monte Carlo estimation \citep{hastings1970monte, del2006sequential}, or by learning amortized samplers, including GFlowNets \citep{bengio2021flow, bengio2023gflownet, choi2025reinforced} and diffusion samplers \citep{zhang2021path, akhound2024iterated, havens2025adjoint, havens2026flow}. The sampling perspective also applies to alignment: the solution of KL-regularized reward maximization is equivalent to sampling from a reward-tilted unnormalized density \citep{korbak2022rl}. Prior methods adopt a sampling perspective to align generative models \citep{venkatraman2024amortizing, domingo2025adjoint, lee2026aligning} using explicit reward evaluations. FestDPO instead learns a few-step implicit sampler that approximates the reward-tilted distribution from a fixed preference dataset, without evaluating a reward function during training.

\section{Background}
\subsection{Preference Optimization}
Given an input $x$, let $y_w \succ y_l\mid x$ denote a preference for $y_w$ over $y_l$. Learning from preferences \citep{christiano2017deep,ziegler2019fine} aims to align model behavior with the latent reward $r^\star$ that governs the preference outcomes. The Bradley–Terry (BT) model \citep{bradley1952rank} formulates how the latent reward $r^\star$ determines preference distribution $p^\star$:
\begin{align} \label{eq: BT likelihood}
p^{\star}\left(y_w \succ y_l \mid x\right)&=\sigma\left(r^{\star}\left(x, y_w\right)-r^{\star}\left(x, y_l\right)\right),
\end{align}
where $\sigma$ is the logistic function. Given a pairwise preference dataset $\mathcal D$, the standard preference optimization objective is the expected reward with a reverse KL divergence penalty to a reference generative model to mitigate reward over-optimization \citep{gao2023scaling}:
\begin{align} \label{eq: rlhf objective}
\mathcal J(\pi)= \mathbb{E}_{x \sim \mathcal{D}, y \sim \pi(\cdot \mid x)}\left[r^\star(x, y)\right]-\beta \mathbb{D}_{\mathrm{KL}}\left[\pi(y \mid x) \| \pi_{\mathrm{ref}}(y \mid x)\right] .
\end{align}
where $\beta>0$ controls the regularization strength and $\pi$, $\pi_{\text{ref}}$ denotes a policy and fixed reference policy respectively. Early approaches \citep{stiennon2020learning,ouyang2022training} train a reward model $\hat r$ to approximate the latent reward $r^\star$ by maximizing the likelihood in \Cref{eq: BT likelihood}. The policy is then optimized using proximal policy optimization (PPO) \citep{schulman2017proximal} with the learned reward model $\hat r$ instead of the latent reward $r^\star$. While effective at aligning models with human preferences, online RL with $\hat r$ requires a separate stage to train the reward model and suffers from the instability and hyperparameter sensitivity of online RL \citep{rafailov2023direct}.

\subsection{Direct Preference Optimization}
DPO \citep{rafailov2023direct} avoids training a separate reward model and an online RL loop by optimizing the generative model directly from preference pairs. DPO objective is derived from the closed-form solution to the KL-regularized objective \citep{peng2019advantage,korbak2022rl}:
\begin{align} \label{eq: optimal distribution}
\pi^\star(y\mid x)
&=\frac1{Z(x)}{\pi_{\mathrm{ref}}(y\mid x)}
\exp\!\left(\frac 1\beta r^\star(x,y)\right),
\end{align}
where $Z(x)=\int \pi_{\mathrm{ref}}(y\mid x)\exp(\frac 1\beta r^\star(x,y))dy$ is the normalizing constant. Then, the latent reward $r^*$ can be expressed as follows:
\begin{align} \label{eq: implicit reward}
r^\star(x, y)=\beta \log \frac{\pi^\star(y \mid x)}{\pi_{\mathrm{ref}}(y \mid x)}+\beta \log Z(x) .
\end{align}
Substituting \Cref{eq: implicit reward} into the \Cref{eq: BT likelihood} cancels the intractable $Z(x)$ and yields the DPO objective for a parameterized policy $\pi_\theta$:
\begin{align}
\label{eq: dpo loss}
\mathcal{L}_\text{DPO}\left(\theta\right)=-\mathbb{E}_{\left(x, y_w, y_l\right) \sim \mathcal{D}}\left[\log \sigma\left(\beta \log \frac{\pi_\theta\left(y_w \mid x\right)}{\pi_{\mathrm{ref}}\left(y_w \mid x\right)}-\beta \log \frac{\pi_\theta\left(y_l \mid x\right)}{\pi_{\mathrm{ref}}\left(y_l \mid x\right)}\right)\right].
\end{align}
Minimizing \Cref{eq: dpo loss} trains $\pi_\theta$ to match the optimal reward-tilted distribution in \Cref{eq: optimal distribution} (see Appendix~\ref{app:optimal distribution} for details).



\section{Extend DPO towards Few-step Generative Models}
While DPO offers a promising way to align generative models with preference pairs, its extension to few-step generators remains underexplored. A main challenge is that many few-step generators are implicit models, for which the likelihood evaluation required by the DPO loss is intractable \citep{ai2026joint}. Existing DPO extensions for diffusion models rely on Evidence Lower Bound (ELBO)-based likelihood surrogates \citep{wallace2024diffusion, yang2024using} or tractable likelihood evaluation of the denoising transition \citep{honavar2025dspo}, which is usually unavailable for few-step implicit generators.

Consequently, extending DPO to few-step generative models requires a new formulation. Given fast sampling and algorithmic diversity, we found that employing sample-based nonparametric estimation to construct a tractable surrogate objective of DPO can lead to the following advantages: (i) making effective use of the computational efficiency of few-step sampling; (ii) being agnostic to the pre-training objective and sampling procedure. These observations lead us to propose \textbf{Few-step DPO (FestDPO)}, a sample-based approximation of DPO for few-step generative models.


\section{Methods}
In this section, we introduce \textbf{Few-step DPO (FestDPO)}, a direct preference optimization method tailored to few-step generative models. \Cref{fig:main} provides an overview of FestDPO: using samples from the trainable generator and the reference generator, we estimate the intractable likelihoods at $y_w$ and $y_l$ through a sample-based approach. Then, we employ estimated likelihoods to approximate DPO objective, leading to the FestDPO objective in \Cref{eq:FestDPO loss}. Theoretically, we establish conditions for the FestDPO loss to converge in probability to the DPO loss, and show that the unique minimizer of the FestDPO loss matches the reward-tilted distribution.



\subsection{FestDPO}
Following \cite{rafailov2023direct}, our goal is to fine-tune a few-step generator to sample from the reward-tilted unnormalized density defined by the latent reward $r^\star$: 
\begin{align} \label{eq : target distribution} 
\pi^\star(y \mid x) \propto \pi_{\text{ref}}(y \mid x) \exp\left(r^\star(x,y)/\beta\right). \end{align} 
To learn a sampler for \Cref{eq : target distribution} without exact likelihood evaluation, we approximate the likelihood terms in the DPO loss \Cref{eq: dpo loss} using sample-based nonparametric density estimation \citep{silverman2018density,liu2022flow}. Thanks to fast sampling from few-step generative models, sample-based estimation becomes practically feasible. For simplicity, we use kernel density estimation (KDE) to approximate the likelihood $\pi_\theta(y\mid x)$ as follows:
\begin{align} 
\label{eq : density approximation}
\hat\pi_{\theta}(y\mid x)=\mathbb E_{y'\sim \pi_\theta}[ k_h\left(y, y'\right)],
\end{align}
where $k_h$ is a Gaussian kernel with bandwidth $h$. We estimate the reference density $\pi_\text{ref}$ analogously. Replacing every likelihood term in \Cref{eq: dpo loss} with \Cref{eq : density approximation} yields a  FestDPO loss: a tractable surrogate of DPO loss for few-step implicit generators:
\begin{align}
\label{eq:FestDPO loss}
\mathcal{L}_{\mathrm{FestDPO}}(\theta)
= -\mathbb{E}_{{(x,y_w,y_l)\sim\mathcal{D} }}
\left[
\log \sigma \left(
\beta \log \frac{\hat{\pi}_{\theta}(y_w \mid x)}
                   {\hat{\pi}_{\mathrm{ref}}(y_w \mid x)}
-
\beta \log \frac{\hat{\pi}_{\theta}(y_l \mid x)}
                   {\hat{\pi}_{\mathrm{ref}}(y_l \mid x)}
\right)
\right].
\end{align}
In practice, we draw $N$ independent noise $\{\epsilon_n\}^N_{n=1}$ and pass them through both the trainable generator and the frozen reference generator simultaneously. We use the resulting samples to construct density estimates and evaluate the FestDPO loss on minibatches of preference pairs. Gradients propagate through $f_\theta$, while $f_{\mathrm{ref}}$ remains fixed. The resulting gradient signal steers the current generator, pulling it toward the preferred samples while pushing it away from the unpreferred samples.

\begin{figure}
    \centering
    \includegraphics[width=\linewidth]{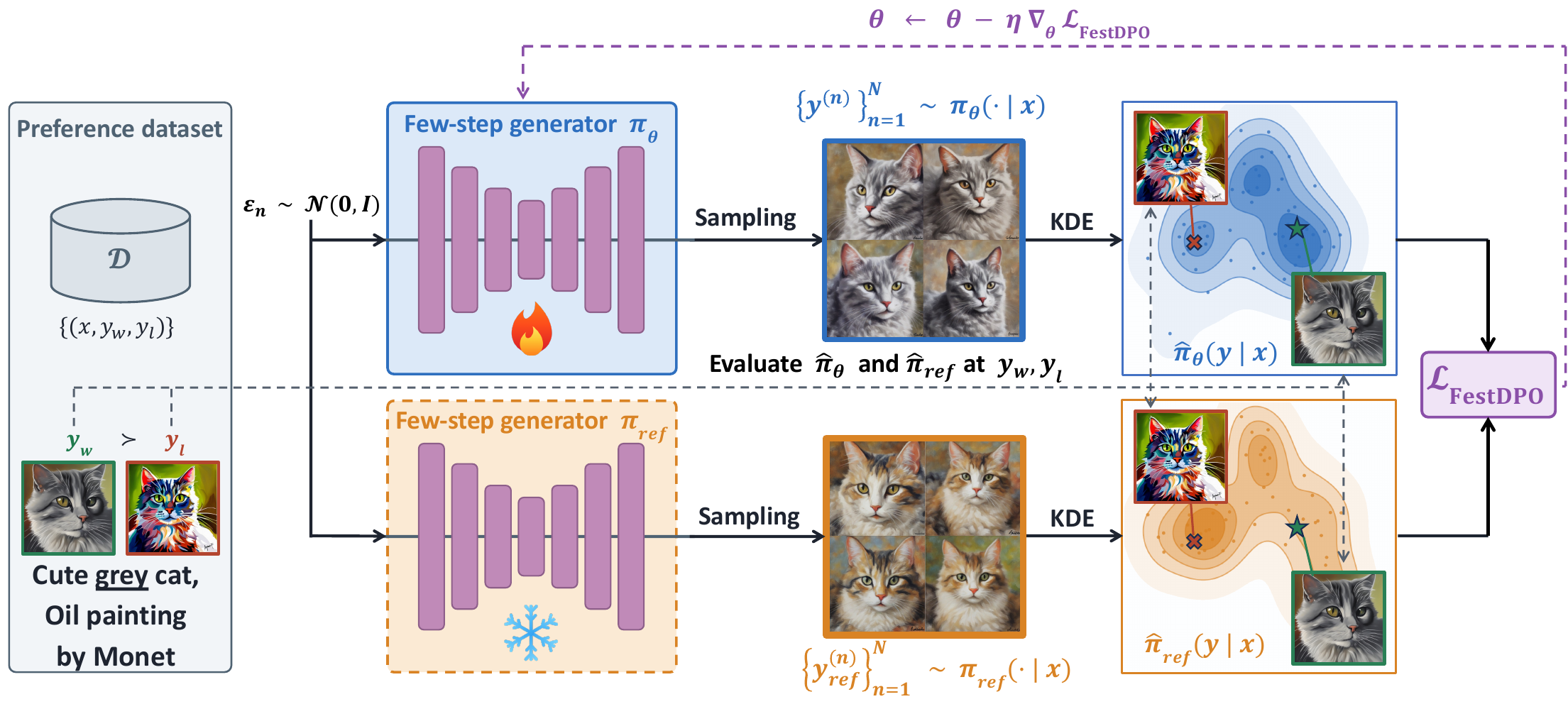}
    \caption{\textbf{Overview of FestDPO.} For a prompt $x$ and a preference pair $(y_w \succ y_l)$, we draw $N$ samples from both the trainable policy $\pi_\theta$ and the frozen reference policy $\pi_{\mathrm{ref}}$. We then construct kernel density estimates $\hat{\pi}_\theta$ and $\hat{\pi}_{\mathrm{ref}}$ and evaluate them at $y_w$ and $y_l$. These estimates yield a tractable surrogate for the DPO loss (\Cref{eq:FestDPO loss}), which we minimize to align $\pi_\theta$ with the preference data.}
    \label{fig:main}
\end{figure}





\subsection{Theoretical Analysis}
\label{sec:theorectical}
In this section, we theoretically analyze the connection between FestDPO and DPO. Under the consistency of kernel density estimation, we show that the FestDPO objective asymptotically recovers the DPO objective. We further establish that minimizers of the FestDPO objective are consistent for the DPO minimizer.

\begin{proposition}[Asymptotic Equivalence of FestDPO and DPO]
\label{eq:proposition1}
Let $\hat{\pi}_{\theta}$ and $\hat{\pi}_{\text{ref}}$ be KDEs constructed from $N$ samples with bandwidth $h$, using a kernel $K$ satisfying the KDE consistency conditions \citep{silverman2018density}. Suppose that $h\to0$ and $Nh^d\to\infty$ as $N\to\infty$, and that, for almost every $(x,y_w,y_l)$ under $\mathcal D$, 
$\pi_\theta(\cdot\mid x)$ and $\pi_{\mathrm{ref}}(\cdot\mid x)$ are continuous and strictly positive at $y_w$ and $y_l$. 
If the per-example FestDPO losses are uniformly integrable, then
\begin{equation}
    \mathcal{L}_{\mathrm{FestDPO}}(\theta)
    \rightarrow
    \mathcal{L}_{\mathrm{DPO}}(\theta).
\end{equation}
\end{proposition}
\Cref{eq:proposition1} shows that, as the number of samples from the model and reference increases, the FestDPO objective converges to the DPO objective. We further show that FestDPO minimizer converges to the DPO minimizer, which corresponds to the optimal target distribution induced by the KL-regularized reward maximization objective in Appendix~\ref{app:optimal distribution}.

\begin{proposition}[Consistency of the FestDPO Minimizer]
\label{eq:proposition2}
Let $(\Theta,d)$ be a metric space. Suppose that the convergence in \Cref{eq:proposition1} holds uniformly over $\Theta$ and that the population DPO objective admits a unique, well-separated minimizer $\theta^\star$. Then, for any sequence of $o_p(1)$-approximate minimizers $\hat{\theta}^{\star}$ of the FestDPO objective,
\begin{equation}
    d(\hat{\theta}^{\star},\theta^\star)
    \xrightarrow{p}0,
    \qquad N\to\infty.
\end{equation}
\end{proposition}

The detailed derivations for Proposition~\ref{eq:proposition1} and \ref{eq:proposition2} are provided in Appendix~\ref{app:proof of proposition1} and \ref{app:proof of proposition2}, respectively.

\subsection{Kernel Density Estimation in Feature Space}
Kernel similarities in raw space may poorly capture the image semantics \citep{theis2015note} or depend on arbitrary rotations and translations of protein structures \citep{jumper2021highly, satorras2021n}. Motivated by kernel-based distribution
matching in embedding spaces \citep{binkowski2018demystifying,deng2026generative,lee2026aligning}, we compute kernel similarity in the embedding space of a pretrained encoder $\phi$ to better reflect semantic and geometric similarities. Specifically, we replace $\hat{\pi}_\theta(\cdot \mid x)$ in \Cref{eq:FestDPO loss} with
\begin{align}
\hat{\pi}_{\theta,\phi}(y\mid x)=\mathbb E_{y'\sim \pi_\theta(\cdot\mid x)}\left[ k_h\left(\phi(y), \phi\left(y'\right)\right)\right] 
\end{align}
\section{Experiments}
In this section, we evaluate FestDPO in three tasks: a 1D toy setting with diverse few-step generative models, text-to-image generators, and protein backbone generation. These experiments demonstrate that FestDPO is capable of sampling from preference-tilted distributions across tasks, from one-dimensional data to high-resolution images and protein structures on SE(3) manifold. Implementation details and hyperparameters for all experiments are provided in \Cref{app:experiments details}
\subsection{Toy Experiments} 
\label{sec: toy experiments}

\begin{figure}
    \centering
    \includegraphics[width=1\linewidth]{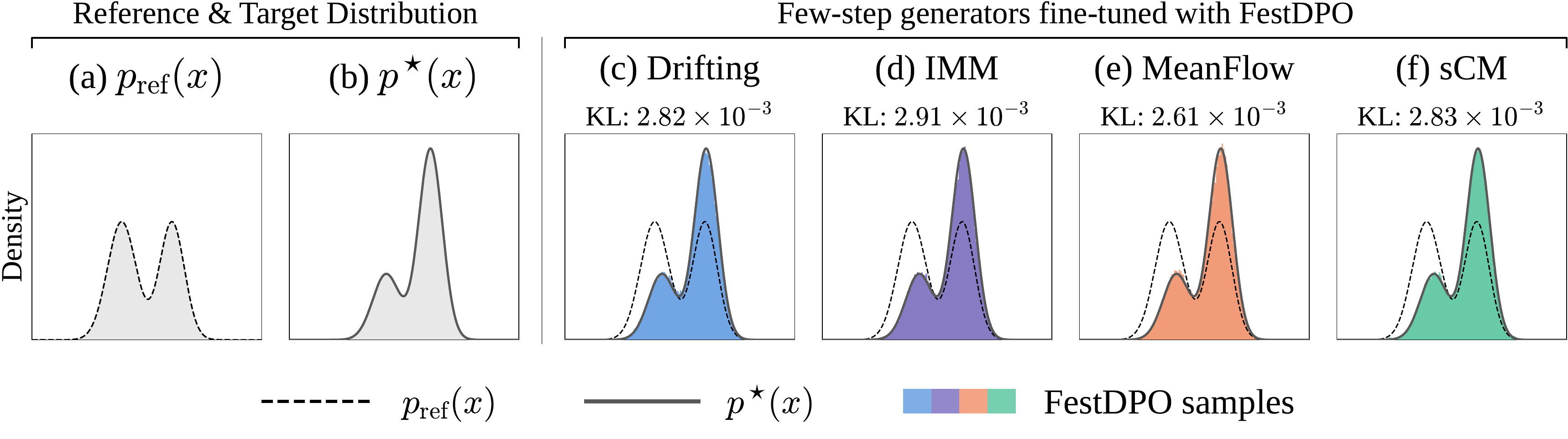}
    \captionsetup{skip=2pt}
    \caption{\textbf{FestDPO aligns diverse few-step generators.} (a) The reference distribution $p_{\text{ref}}(x)$ and (b) the reward-tilted target distribution $p^\star(x) \propto p_\text{ref}(x)\exp(r^{\star}(x))$. (c)--(f) Distributions of Drifting, IMM, MeanFlow, and sCM fine-tuned with FestDPO. Across all four generators, the fine-tuned distributions closely match $p^\star(x)$, with KL divergences to the target of approximately $3 \times 10^{-3}$.}
    \label{fig:toy}
    \vspace{-10pt}
\end{figure}

To evaluate the capability of FestDPO in sampling from a tilted target distribution defined by the latent reward model: $p^{\star}(x)\propto p_\text{ref}(x)\exp(r^{\star}(x))$, we perform a synthetic experiment on a 1D Gaussian Mixture Model. We first pre-train four few-step generators, namely Drifting \citep{deng2026generative}, IMM \citep{zhou2025inductive}, MeanFlow \citep{geng2026mean}, and sCM \citep{lu2025simplifying}. We construct an offline preference dataset by drawing pairs of samples from each pre-trained generator and labeling them with the Bradley-Terry model under $r^\star$, and fine-tune each generator with FestDPO on this dataset. As shown in \Cref{fig:toy}, FestDPO closely matches the target distribution across all generators, achieving a KL divergence of approximately $3 \times 10^{-3}$, while preserving the multi-modal structure of the target. Detailed descriptions are provided in Appendix~\ref{app:toy experiments}.

\subsection{Text-to-image generation} 
We evaluate FestDPO for aligning text-to-image models with user preferences. We first describe the experimental setup in Section~\ref{sec: T2I setup}, then report the main results in Section~\ref{sec: T2I main results}, and finally present ablation studies in Section~\ref{sec: T2I ablation}.

\subsubsection{Experimental setup}
\label{sec: T2I setup}

\textbf{Baselines.} We compare FestDPO with PSO~\citep{miao2025tuning} and DrPO~\citep{jiang2026drifting}, which are the preference optimization methods for fine-tuning a few-step generator from offline preference pairs. We defer a detailed description of the baselines to Appendix~\ref{app:baseline-details}.

\textbf{Base Models \& Preference Dataset.} We fine-tune two few-step text-to-image generators: SDXL-Turbo \citep{sauer2024adversarial} and SDXL-DMD2 \citep{yin2024improved}. All methods are trained on the training split of Pick-a-Pic v2~\citep{kirstain2023pick}, an open dataset of text-to-image prompts and preferences of real users over generated images.

\textbf{Evaluation.} We evaluate the generated images using PickScore~\citep{kirstain2023pick} and Aesthetic Score~\citep{schuhmann2023laion}. We also report pairwise win rates against the corresponding base models. We then evaluate on 500 prompts sampled from each of three sets: the held-out Pick-a-Pic v2 test split (PickV2), Parti-Prompts~\citep[P2;][]{yu2022scaling}, and the HPSv2 benchmark~\citep{wu2023human}. We further conduct a human study on 80 of these prompts, reported in \Cref{sec: T2I human evaluations}.

\textbf{Feature Encoder.} Our default feature extractor is the 340M-parameter Latent-MAE encoder \citep{he2022masked} from Drifting Models~\citep{deng2026generative}. To assess the sensitivity of FestDPO to the choice of feature representation, we further conduct ablations using DINOv2 \citep{oquab2023dinov2} and CLIP \citep{radford2021learning} features in \Cref{sec: T2I ablation}. 

\subsubsection{Main Results}
\label{sec: T2I main results}
We report the win rates of FestDPO, PSO, and DrPO over base models on Pick-a-Pic v2, PartiPrompts, and HPSv2 evaluation datasets in \Cref{fig:T2I win rate}. For each base model, we compute the win rate as the percentage of prompts on which the fine-tuned model scores higher than the base model, separately for PickScore and Aesthetic score. FestDPO achieves the highest win rate on every prompt dataset under PickScore and Aesthetic. 
\Cref{fig:T2I anytime rate} further shows the anytime performance on SDXL-DMD2. FestDPO rapidly improves PickScore within the first 150 steps and remains stable, whereas both baselines show slower and smaller improvements.
We report the absolute scores and win rates on Pick-a-Pic v2, PartiPrompts, and HPSv2 evaluation datasets in Appendix~\ref{app:additional experiments} and provide the qualitative comparison in Appendix~\ref{app: Qualitative_results}.


\begin{figure}
    \centering
    \includegraphics[width=1.0\linewidth]{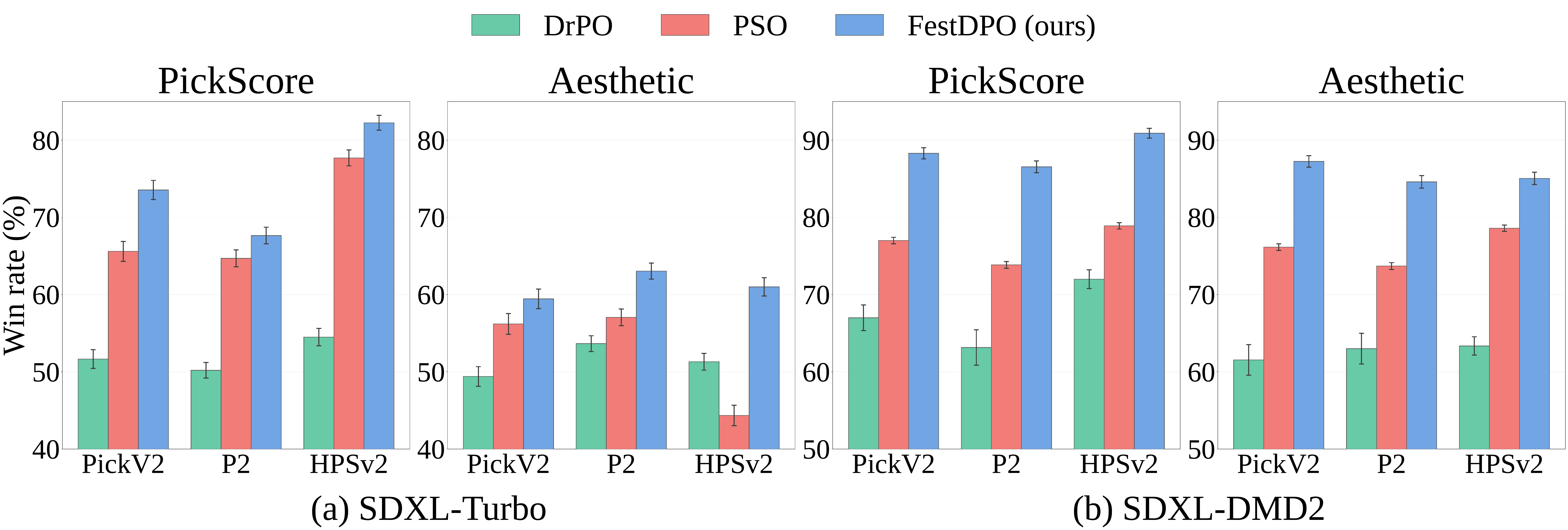}
    \caption{Win rate comparison against (a) SDXL-Turbo and (b) SDXL-DMD2 on Pick-a-Pic v2, Parti-Prompts, and HPSv2 evaluation datasets. Each bar represents the mean win rate (\%) of each method for PickScore and Aesthetic score, and error bars denote the standard deviation.}
    \label{fig:T2I win rate}
    \vspace{-10pt}
\end{figure}

\begin{figure}
    \centering
    \includegraphics[width=0.95\linewidth]{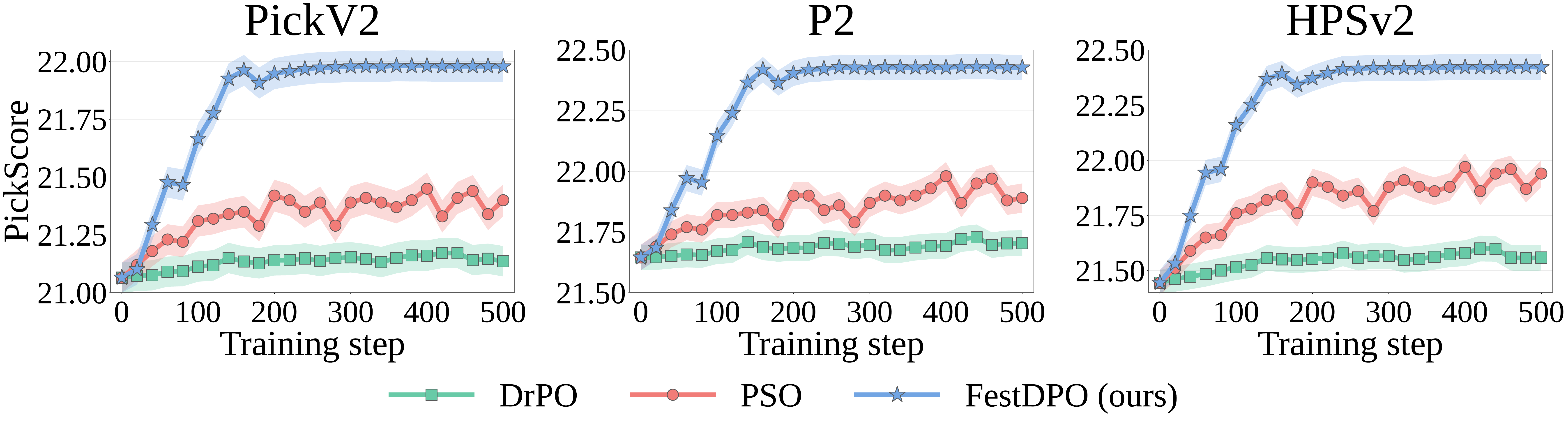}
    \vspace{-5pt}
    \caption{PickScore of DrPO, PSO, and FestDPO during fine-tuning of SDXL-DMD2 on Pick-a-Pic v2, Parti-Prompts, and HPSv2. Lines and shaded regions denote the mean and standard deviation.}
    \label{fig:T2I anytime rate}
    \vspace{-10pt}
\end{figure}

\subsubsection{Ablation Study}
\label{sec: T2I ablation}

In this section, We conduct ablation studies on HPSv2 with SDXL-Turbo. Since FestDPO employs nonparametric density estimation, we first study the effect of the number of samples $N$ used to construct the KDE. As shown in Table~\ref{tab:kde sample}, increasing $N$ from 8 to 24 improves PickScore, while Aesthetic scores remain approximately stable across $N$. These results indicate that FestDPO is robust to the choice of $N$, with even a small number of samples yielding substantial gains over the pretrained model. We further investigate the sensitivity of FestDPO to the feature encoder. \Cref{tab:kde-feature} compares alignment performance using Latent-MAE \citep{deng2026generative}, CLIP \citep{radford2021learning}, and DINOv2 \citep{oquab2023dinov2} as pretrained encoders for density estimation. FestDPO achieves comparable performance across all three encoders, demonstrating that its effectiveness is robust to the choice of feature encoder. Results on Pick-a-Pic v2, Parti-Prompts, and HPSv2 are provided in \Cref{app:additional experiments}.

\begin{table}[t]
\centering
\small
\setlength{\tabcolsep}{5pt}
\begin{minipage}[t]{0.42\linewidth}
\centering
\caption{Sensitivity analysis on $N$.}
\label{tab:kde sample}
\begin{tabular}{lcc}
\toprule
 $N$& PickScore $\uparrow$ & Aesthetic $\uparrow$ \\
\midrule
$N=8$            & 23.06 $\pm$ 1.32 & 6.13 $\pm$ 0.68 \\
$N=12$ (default) & 23.13 $\pm$ 1.35 & 6.11 $\pm$ 0.66 \\
$N=24$           & 23.14 $\pm$ 1.37 & 6.07 $\pm$ 0.63 \\
\bottomrule
\end{tabular}
\end{minipage}
\hfill
\begin{minipage}[t]{0.54\linewidth}
\centering
\caption{Ablation on the feature encoder $\phi$.}
\label{tab:kde-feature}
\begin{tabular}{lcc}
\toprule
Feature Encoder & PickScore $\uparrow$ & Aesthetic $\uparrow$ \\
\midrule
Latent MAE (default) & 23.13 $\pm$ 1.35 & 6.11 $\pm$ 0.66 \\
CLIP                 & 23.17 $\pm$ 1.35 & 6.08 $\pm$ 0.64 \\
DINOv2               & 23.07 $\pm$ 1.34 & 6.08 $\pm$ 0.65 \\
\bottomrule
\end{tabular}
\end{minipage}
\end{table}

\subsubsection{Human Evaluations}
\label{sec: T2I human evaluations}
\begin{wraptable}{r}{0.42\linewidth}
\vspace{-15pt}
\centering
\small
\caption{Human evaluation results (mean rating). Best results are in \textbf{bold}.}
\vspace{-5pt}
\label{tab:human-eval}
\setlength{\tabcolsep}{4pt}
\begin{tabular}{lcc}
\toprule
Method & Align. $\uparrow$ & Aesth. $\uparrow$ \\
\midrule
SDXL-Turbo     & 3.683 & 3.149 \\
PSO            & 3.767 & 3.286 \\
DrPO           & 3.696 & 3.158 \\
FestDPO (ours) & \textbf{4.199} & \textbf{3.860} \\
\bottomrule
\end{tabular}
\vspace{-10pt}
\end{wraptable}
We conduct a human evaluation study to assess the prompt alignment and aesthetic quality of the generated samples. As shown in \Cref{tab:human-eval}, FestDPO receives the highest ratings for both prompt alignment and aesthetics. The results indicate that human annotators consistently preferred images generated by FestDPO over those from the baselines. Details of the human evaluation protocol are provided in Appendix \ref{app:human evaluation}.

\subsection{Protein backbone generation}
We evaluate FestDPO on few-step protein backbone generation \citep{woo2026riemannian} to assess the capability of aligning the few-step generative model on the SE(3) manifold.

\subsubsection{Experiment setup}
\paragraph{Base model and baselines.} 
We fine-tune Riemannian MeanFlow \citep[RMF;][]{woo2026riemannian}, a few-step flow map-based protein backbone generator defined on the $\mathrm{SE}(3)$ manifold. Specifically, we fine-tune the RMF-S variant with single-step generation. We employ DrPO \citep{jiang2026drifting} as a baseline and exclude PSO \citep{miao2025tuning} since extending PSO to flow maps is not trivial.

\paragraph{Evaluation.} 
We optimize two reward functions: (i) self-consistency root mean square distance (scRMSD) for structural designability, and (ii) ratio of $\beta$-sheet secondary structure (SS-match). scRMSD evaluates whether a generated protein backbone is realistic. To compute scRMSD, we inverse-fold the generated structure using ProteinMPNN \citep{dauparas2022robust}, predict the 3D structure of the resulting sequence with ESMFold \citep{lin2023evolutionary}, and calculate the RMSD against the original generated backbone. Motivated by SS rewards in prior works \citep{huguet2024sequence, pmlr-v267-venkatraman25a, su2026iterative}, our second reward encourages the formation of $\beta$-sheet secondary structure within the protein. Secondary structures are commonly grouped into $\alpha$-helices, $\beta$-strands, and coils. We use DSSP \citep{kabsch1983dictionary} and P-SEA \citep{labesse1997p} to assign secondary-structure labels to generated backbones. For each target, we construct a separate offline dataset of 8,000 preference pairs from RMF-S samples, using one-step generation for SS-match and ten-step generation for scRMSD. Within each pair, we label the sample with higher SS-strand reward or lower scRMSD as preferred. Following \citet{woo2026riemannian}, we measure structural similarity between the generated backbones through pairwise TM-score \citep{zhang2004scoring}.

\paragraph{Feature Encoder.} 
We use the pretrained ProteinMPNN encoder \citep{dauparas2022robust} to obtain structural features invariant to global rotations and translations. The encoder remains frozen throughout training and defines the feature space used for kernel density estimation.

\subsubsection{Results}
As illustrated in \Cref{tab:beta-sheet-eval} and \Cref{tab:designability-eval}, FestDPO achieves a higher $\beta$-strand residue fraction and lower scRMSD than both base model and DrPO. FestDPO exhibits better structural diversity than DrPO, although diversity decreases relative to the base model. \Cref{fig:protein} compares backbones generated by the base model and FestDPO after optimization with the SS-match reward, showing increased $\beta$-sheet content in the generated protein structure. 

\begin{table*}[t]
\centering

\begin{minipage}[t]{0.43 \textwidth}
\vspace{0pt}
\centering

\caption{$\beta$-sheet optimization results.}
\label{tab:beta-sheet-eval}
\resizebox{\linewidth}{!}{%
\begin{tabular}{lcc}
\toprule
Method & $\beta$-sheet \% $\uparrow$ & TM-diversity $\downarrow$ \\
\midrule
RMF-S (base)  & 0.153 $\pm$ 0.006
        & \textbf{0.281 $\pm$ 0.006} \\
DrPO    & 0.218 $\pm$ 0.005
        & 0.294 $\pm$ 0.002 \\
FestDPO & \textbf{0.295 $\pm$ 0.014}
        & 0.292 $\pm$ 0.002 \\
\bottomrule
\end{tabular}%
}

\vspace{1.5em}

\caption{scRMSD optimization results.}
\label{tab:designability-eval}
\resizebox{\linewidth}{!}{%
\begin{tabular}{lcc}
\toprule
Method & scRMSD $\downarrow$ & TM-diversity $\downarrow$ \\
\midrule
RMF-S (base) & 9.303 $\pm$ 0.549 & \textbf{0.281 $\pm$ 0.006} \\
DrPO    & 9.010 $\pm$ 0.383  & 0.305 ± 0.001  \\
FestDPO & \textbf{7.435 $\pm$  0.384} &  0.294 ± 0.002  \\
\bottomrule
\end{tabular}%
}
\end{minipage}
\hfill
\begin{minipage}[t]{0.53\textwidth}
\vspace{-10pt}
\centering
\includegraphics[width=\linewidth]{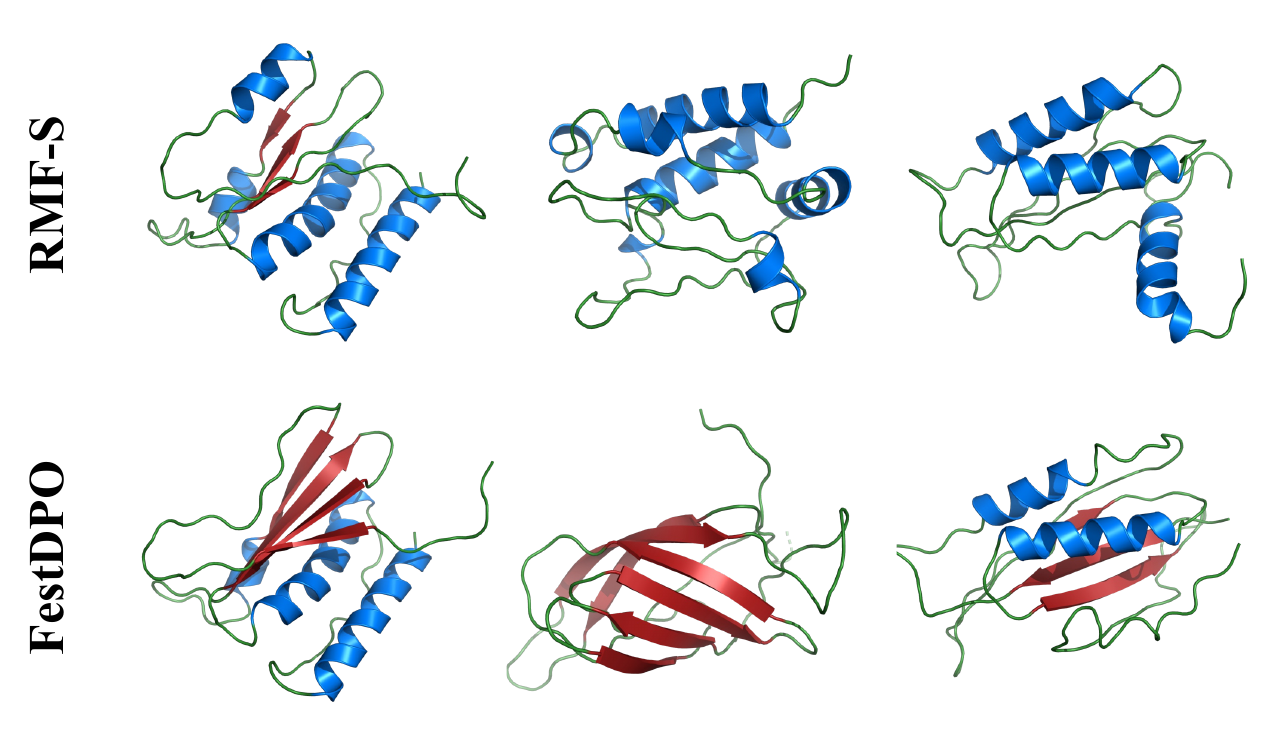}
\vspace{-10pt}
\captionof{figure}{Qualitative comparison between generated protein backbone structures.}
\label{fig:protein}
\end{minipage}

\end{table*}
\section{Discussion}
\paragraph{Conclusion.}
We introduce FestDPO, a sample-based direct preference optimization method for few-step generative models. FestDPO approximates DPO objective with sample-based nonparametric estimation, exploiting the fast sampling speed of few-step generative models. We theoretically show that the FestDPO objective converges in probability to the DPO objective as the number of samples grows, and that its minimizer is consistent. In a 1-D toy experiment, FestDPO learns to sample from the unnormalized target density across diverse few-step generative models, showing the agnosticism of FestDPO. In text-to-image generation, FestDPO achieves the highest win rates against the base model across Pick-a-Pic v2, Parti-Prompts, and HPSv2 datasets and receives the highest human ratings. In protein backbone generation, FestDPO achieves a higher $\beta$-strand fraction and lower scRMSD than both the base model and baseline.

\paragraph{Limitations.}
FestDPO employs sample-based nonparametric density estimation, which is asymptotically consistent as the number of samples grows. In practice, however, we empirically show that sample-based estimation with finite samples is effective for preference optimization, also in high-dimensional domains such as images and proteins. Since density estimation directly in the high-dimensional data space is challenging, FestDPO performs it in the feature space of a pretrained encoder, which captures the underlying structure of the data. While this design introduces a dependence on the choice of feature encoder, our ablation study shows that FestDPO remains robust across different pretrained encoders.


\clearpage
\bibliography{iclr2027_conference}
\bibliographystyle{iclr2027_conference}

\clearpage

\appendix

\section{Optimal Distribution of the KL-Regularized Objective}
\label{app:optimal distribution}
Following \citet{rafailov2023direct}, we first derive the optimal policy of the KL-regularized reward maximization objective. 
\begin{equation}
\label{app:kl objective}
    \max _\pi \mathbb{E}_{x \sim \mathcal{D}, y \sim \pi(\cdot \mid x)}[r(x, y)]-\beta \mathbb{D}_{\mathrm{KL}}\left[\pi(y \mid x) \| \pi_{\mathrm{ref}}(y \mid x)\right]
\end{equation}
where $r^{\star}(x,y)$ denotes the latent reward function, $\pi_{\mathrm{ref}}$ the reference policy, and $\pi$ the policy to be optimized.
For $\beta>0$, the objective in \Cref{app:kl objective} can be rewritten as
\begin{align}
    \max _\pi \mathbb{E}_{x \sim \mathcal{D}, y \sim \pi(\cdot \mid x)} & {[r(x, y)]-\beta \mathbb{D}_{\mathrm{KL}}\left[\pi(y \mid x) \| \pi_{\mathrm{ref}}(y \mid x)\right] } \\
    & =\max _\pi \mathbb{E}_{x \sim \mathcal{D}} \mathbb{E}_{y \sim \pi(\cdot \mid x)}\left[r(x, y)-\beta \log \frac{\pi(y \mid x)}{\pi_{\mathrm{ref}}(y \mid x)}\right] \\
    & =\min _\pi \mathbb{E}_{x \sim \mathcal{D}} \mathbb{E}_{y \sim \pi(\cdot \mid x)}\left[\log \frac{\pi(y \mid x)}{\pi_{\mathrm{ref}}(y \mid x)}-\frac{1}{\beta} r(x, y)\right] \\
    & =\min _\pi \mathbb{E}_{x \sim \mathcal{D}} \mathbb{E}_{y \sim \pi(\cdot \mid x)}\left[\log \frac{\pi(y \mid x)}{\frac{1}{Z(x)} \pi_{\mathrm{ref}}(y \mid x) \exp \left(\frac{1}{\beta} r(x, y)\right)}-\log Z(x)\right]
\end{align}
where the partition function denotes $Z(x)=\int\pi_{\mathrm{ref}}(y \mid x)\exp\left(\frac{1}{\beta}r^{\star}(x,y)\right)dy$. Because the partition function does not depend on the policy $\pi$, we can define the closed-form solution as
\begin{equation}
\label{eq:kl optimal distribution}
    \pi^\star(y \mid x)=\frac{1}{Z(x)} \pi_{\mathrm{ref}}(y \mid x) \exp \left(\frac{1}{\beta} r(x, y)\right).
\end{equation}
Rearranging \Cref{eq:kl optimal distribution}, the reward can be expressed in terms of the corresponding optimal policy:
\begin{equation}
\label{eq:app_implicit_reward}
    r^{\star}(x,y)    =    \beta    \log    \frac{\pi^\star(y \mid x)}    {\pi_{\mathrm{ref}}(y \mid x)}    +    \beta \log Z(x).
\end{equation}

Under the Bradley--Terry preference model, the probability that $y_w$ is preferred over $y_l$ is given by
\begin{equation}
\label{eq:app_bt}
    P(y_w \succ y_l \mid x)    =    \sigma\left(    r^{\star}(x,y_w)-r^{\star}(x,y_l)    \right).
\end{equation}
Substituting \Cref{eq:app_implicit_reward} into \Cref{eq:app_bt}, the partition function $Z(x)$ cancels, yielding
\begin{equation}
\label{eq:app_dpo_preference}
    P(y_w \succ y_l \mid x)    =    \sigma\left[    \beta    \log    \frac{\pi^\star(y_w \mid x)}    {\pi_{\mathrm{ref}}(y_w \mid x)}
    -    \beta    \log    \frac{\pi^\star(y_l \mid x)}    {\pi_{\mathrm{ref}}(y_l \mid x)}    \right].
\end{equation}
DPO objective is then obtained by parameterizing this preference probability and maximizing its likelihood on the preference dataset. Therefore, under the Bradley--Terry model, DPO objective and the KL-regularized reward maximization objective share the same optimal policy, given by the reward-tilted distribution in \Cref{eq:kl optimal distribution} \citep{rafailov2023direct}.

\section{proof of \Cref{eq:proposition1}}
\label{app:proof of proposition1}
\begin{proof}
For simplicity, define the normalized KDEs using $N$ samples:
\begin{align}
\hat{\pi}_{\theta}(y \mid x)=\frac{1}{N}\sum_{n=1}^N k_h\left(y, y^{(n)}\right), \quad 
\hat{\pi}_{\text{ref}}(y \mid x)=\frac{1}{N}\sum_{n=1}^N k_h\left(y, y^{(n)}_{\text{ref}}\right),
\end{align}
where $y^{(n)}\sim \pi_\theta(\cdot \mid x)$ and $ y^{(n)}_{\text{ref}}\sim \pi_{\text{ref}}(\cdot \mid x)$, and $k_h(y,z)=h^{-d}K((y-z)/h)$ for $y, z \in \mathbb{R}^d$, where $K$ is a strictly positive bounded Borel kernel satisfying the standard conditions for pointwise KDE consistency \citep{devroye1985nonparametric,silverman2018density}.

As $N \to \infty$, suppose the bandwidth $h$ satisfies  $h \to 0$, $Nh^d \to \infty$. Under these conditions, for almost every $(x,y_w,y_l) \sim \mathcal D$ at which the corresponding densities are continuous  and strictly positive at $y_w$ and $y_l$, the KDE consistency property  \citep{devroye1985nonparametric,silverman2018density} gives
\begin{align}
    \hat{\pi}_{\theta}(y \mid x)
    &\xrightarrow{p}
    \pi_{\theta}(y \mid x) \ \ \text{and} \ \    \hat{\pi}_{\text{ref}}(y \mid x)
    \xrightarrow{p} 
    \pi_{\text{ref}}(y \mid x), \quad \text{as} \ \ N\rightarrow\infty.
\end{align}

By the continuous mapping theorem \citep{van2000asymptotic}, the log-density ratios satisfy
\begin{align}
    \log\frac{\hat{\pi}_{\theta}(y_w\mid x)}{\hat{\pi}_{\mathrm{ref}}(y_w\mid x)}
    &\xrightarrow{p}
    \log\frac{{\pi}_\theta(y_w\mid x)}{\pi_{\mathrm{ref}}(y_w\mid x)}, \quad
    \log\frac{\hat{\pi}_{\theta}(y_l\mid x)}{\hat{\pi}_{\mathrm{ref}}(y_l\mid x)}
    \xrightarrow{p}
    \log\frac{\pi_\theta(y_l\mid x)}{\pi_{\mathrm{ref}}(y_l\mid x)},
\end{align}
Combining the above,
\begin{align}
    &
    \beta
    \log
    \frac{\hat{\pi}_{\theta }(y_w\mid x)}{\hat{\pi}_{\mathrm{ref} }(y_w\mid x)}-\beta\log\frac{\hat{\pi}_{\theta }(y_l\mid x)}{\hat{\pi}_{\mathrm{ref} }(y_l\mid x)}
    \nonumber
    \quad\xrightarrow{p}\quad
    \beta\log\frac{\pi_\theta(y_w\mid x)}{\pi_{\mathrm{ref}}(y_w\mid x)}-\beta\log\frac{\pi_\theta(y_l\mid x)}{\pi_{\mathrm{ref}}(y_l\mid x)}.
\end{align}

Since $z\mapsto-\log\sigma(z)$ is continuous, the continuous mapping theorem yields
\begin{align}
    {\ell}_{\mathrm{FestDPO}}(x,y_w,y_l)
    \xrightarrow{p}
    \ell_{\mathrm{DPO}}(x,y_w,y_l).
\end{align}
By assumption, the sequence of per-example FestDPO losses is uniformly integrable. For instance, this assumption holds if the corresponding log-density ratios are uniformly bounded in $N$, which is guaranteed when the estimated densities are uniformly bounded above and bounded away from zero at the evaluated samples. Together with the convergence in probability established above, this implies convergence in $L^1$ \citep{pollard2002user}:
\begin{equation}
\mathbb{E}\left[
\left|
\ell_{\mathrm{FestDPO}}(x,y_w,y_l)
-
\ell_{\mathrm{DPO}}(x,y_w,y_l)
\right|
\right]
\longrightarrow 0.
\end{equation}
Using the inequality $|\mathbb{E}[X]|\leq \mathbb{E}[|X|]$, we then obtain
\begin{align}
    \left|
    \mathcal{L}_{\mathrm{FestDPO}}(\theta)
    -
    \mathcal{L}_{\mathrm{DPO}}(\theta)
    \right|
    &=
    \left|
    \mathbb{E}\left[
    \ell_{\mathrm{FestDPO}}
    -
    \ell_{\mathrm{DPO}}
    \right]
    \right| \\
    &\leq
    \mathbb{E}\left[
    \left|
    \ell_{\mathrm{FestDPO}}
    -
    \ell_{\mathrm{DPO}}
    \right|
    \right]
    \longrightarrow 0.
\end{align}
Hence, the FestDPO objective converges to the DPO objective:
\begin{equation}
    \mathcal{L}_{\mathrm{FestDPO}}(\theta)
    \longrightarrow
    \mathcal{L}_{\mathrm{DPO}}(\theta).
\end{equation}

\end{proof}

\section{proof of \Cref{eq:proposition2}}
\label{app:proof of proposition2}

\begin{proof}
Let $\theta^\star = \arg\min_{\theta\in\Theta} \mathcal{L}_{\mathrm{DPO}}(\theta)$ denote the minimizer of DPO objective. By the assumed uniform version of \Cref{eq:proposition1} and the well-separatedness assumption, we have
\begin{equation}
    \label{app:uniform version of proposition1}
    \sup_{\theta\in\Theta}
    \left|
    \mathcal{L}_{\mathrm{FestDPO}}(\theta)
    -
    \mathcal{L}_{\mathrm{DPO}}(\theta)
    \right|
    \xrightarrow{p}0,
\end{equation}
and, for every $\epsilon>0$,
\begin{equation}
    \inf_{\theta:\,d(\theta,\theta^\star)\geq\epsilon}
    \mathcal{L}_{\mathrm{DPO}}(\theta)
    >
    \mathcal{L}_{\mathrm{DPO}}(\theta^\star),
\end{equation}
where the second condition is the minimization analogue of the condition in \citet[Theorem~5.7]{van2000asymptotic}.
Since $\hat\theta^{\star}$ is an $o_p(1)$-approximate minimizer of FestDPO,
\begin{equation}
\mathcal{L}_{\mathrm{FestDPO}}(\hat{\theta}^{\star}) \leq \mathcal{L}_{\mathrm{FestDPO}}(\theta^{\star})+o_p(1) 
\end{equation}
where $o_p(1)$ denotes a sequence of random variables that converges to zero in probability as $N\to\infty$. 
By \Cref{app:uniform version of proposition1} above,
\begin{equation}
\mathcal{L}_{\mathrm{FestDPO}}(\hat{\theta}^{\star}) \leq \mathcal{L}_{\mathrm{DPO}}(\theta^{\star})+o_p(1) 
\end{equation}
\begin{align}
    \mathcal{L}_{\mathrm{DPO}}(\hat{\theta}^{\star})- \mathcal{L}_{\mathrm{DPO}} (\theta^{\star})
    & \leq  \mathcal{L}_{\mathrm{DPO}}(\hat{\theta}^{\star})- \mathcal{L}_{\mathrm{FestDPO}}(\hat{\theta}^{\star})+o_p(1) \\
    & \leq \sup _\theta    \left|
    \mathcal{L}_{\mathrm{FestDPO}}(\theta)
    -
    \mathcal{L}_{\mathrm{DPO}}(\theta)
    \right|+o_p(1) \xrightarrow{\mathrm{P}} 0 .
\end{align}
Since $\theta^\star$ minimizes the population DPO objective, the left-hand side is nonnegative. Hence,
\begin{equation}
    \mathcal{L}_{\mathrm{DPO}}(\hat{\theta}^{\star})
    -
    \mathcal{L}_{\mathrm{DPO}}(\theta^\star)
    \xrightarrow{p}0.
\end{equation}
By the well-separatedness condition, this implies \citep{van2000asymptotic}
\begin{equation}
    d(\hat{\theta}^{\star},\theta^\star)
    \xrightarrow{p}0.
\end{equation}
\end{proof}

\section{Baseline Details}
\label{app:baseline-details}

\subsection{PSO}
\label{app:pso}

PSO \citep{miao2025tuning} fine-tunes a distilled model $p_\theta(x_0 \!\mid\! c)$ of $N = 1 \sim 4$ steps. Since applying diffusion loss makes blurry outputs, it  maximizes a relative margin between a target sample $x^\tau_0 \sim p_{\mathrm{data}}$ and a reference sample $x^\rho_0 \sim p_\theta$ from the model being tuned, regularized by the pretrained distilled model $p_{\mathrm{pre}}$ with weight $\beta$:
\begin{equation}
\label{eq:pso loss}
  \mathcal{L}
  = -\,\mathbb{E}\,\left[
    \log \sigma\!\left(
      \beta \log \frac{p_\theta(x^\tau_0 \mid c)}{p_{\mathrm{pre}}(x^\tau_0 \mid c)}
      - \beta \log \frac{p_\theta(x^\rho_0 \mid c)}{p_{\mathrm{pre}}(x^\rho_0 \mid c)}
    \right)\right],
\end{equation}
where $c$ is condition. In \Cref{eq:pso loss}, Evaluating $p_\theta(x_0 \mid c)$ would require marginalizing over the intermediate states which is generally intractable. PSO therefore focus on the relative margin to joint likelihoods of whole trajectories. Specifically, the forward diffusion process yields the data trajectory, and the reverse generative process yields the reference one:
\begin{equation}
  \mathcal{L}_{\mathrm{PSO}}
  = -\,\mathbb{E}\!\left[\log \sigma\!\left(\beta \sum_n \left(
      \log \frac{p_\theta(x^\tau_{t_{n-1}} \mid x^\tau_{t_n}, c)}{p_{\mathrm{pre}}(x^\tau_{t_{n-1}} \mid x^\tau_{t_n}, c)}
      - \log \frac{p_\theta(x^\rho_{t_{n-1}} \mid x^\rho_{t_n}, c)}{p_{\mathrm{pre}}(x^\rho_{t_{n-1}} \mid x^\rho_{t_n}, c)}
    \right)\right)\right].
  \label{eq:pso-traj}
\end{equation}
PSO makes these transition densities explicit by casting few-step denoising as an MDP with state $(x_{t_n}, t_n)$, action $x_{t_{n-1}}$, and Gaussian policy $\mathcal{N}(\mu_\theta, \sigma^2_{t_n} I)$. By substituting the MDP action-state conditional distribution, final objective for PSO is
\begin{align}
\mathcal{L}_{\text {PSO }}=-\mathbb{E} & {\left[\operatorname { l o g } \sigma \left(-\beta \cdot \sum_{n=2}^N\left(\left(\left\|\epsilon^\tau-\epsilon_\theta\left(x_{t_n}^\tau, t_n, c\right)\right\|^2-\left\|\epsilon^\tau-\epsilon_{\text {pre }}\left(x_{t_n}^\tau, t_n, c\right)\right\|^2\right)\right.\right.\right.} \notag \\
& \left.\left.\left.-\frac{1}{2 \sigma_{t_n}^2}\left(\left\|x_{t_{n-1}}^\rho-\mu_\theta\left(x_{t_n}^\rho, t_n, c\right)\right\|^2-\left\|x_{t_{n-1}}^\rho-\mu_{\text {pre }}\left(x_{t_n}^\rho, t_n, c\right)\right\|^2\right)\right)\right)\right].
  \label{eq:pso-final}
\end{align}
In our offline setting, the reference samples are pre-collected offline dataset instead of being drawn from the current model, and their trajectories are approximated by the diffusion forward process. A detailed explanation is provided in \cite{miao2025tuning}.

\subsection{DrPO}
\label{app:drpo}
DrPO \citep{jiang2026drifting} aligns a deterministic one-step generator $x = g_\theta(\epsilon, c)$ by kernel $k$.  Given positive $\mathcal{A^+}$ and negative feature samples $\mathcal{A^-}$, they define a non-parametric (dipole) reward with  strength $\gamma$:
\begin{equation}
    R_{\text {dipole }}(z)=\exp (E(z)), \quad E(z)=\gamma \sum_{j=1}^M\left[k\left(z, a_j^{+}\right)-k\left(z, a_j^{-}\right)\right],
  \label{eq:dipole}
\end{equation}
where  $a^+_j\in\mathcal{A^+}$ and $a^-_j\in\mathcal{A^-}$ denotes positive and negative sample. Assuming non-negative differentable reward $R(z)$, given $x=g_\theta(\epsilon,c)$ and $z=\phi(x)$, the pathwise gradient of the reference-regularized soft RL objective is
\begin{equation}
  \nabla_\theta J
  = \mathbb{E}_\epsilon\!\left[
      \Big(\nabla_z \log R(z)
      + \lambda \big(\nabla_z \log p_{\mathrm{pre}}(z) - \nabla_z \log p_\theta(z)\big)\Big)
      \nabla_\theta z \right],
  \label{eq:drpo-pathwise}
\end{equation}
where $\nabla_\theta z=\nabla_x\phi(x) \nabla_\theta g_\theta (\epsilon,c)$ and $\phi$ denotes the feature encoder. The reward score term reduces to $V_{\mathrm{pref}} = \nabla_z \log R_{\mathrm{dipole}}(z)$, and the reference term can be represented by a mean-shift difference between reference features $\mathcal{R}$ and current model features $\mathcal{Z}$. Setting $V_{\mathrm{DrPO}} = V_{\mathrm{pref}} + \lambda (\hat{\mu}_{\mathcal{R}} - \hat{\mu}_{\mathcal{Z}})$,
the generator is trained as in drifting models by regressing each feature onto a stop-gradient target with velocity scale $\eta$:
\begin{equation}
  z^\star_i = \mathrm{sg}\big(z_i + \eta\, V_{\mathrm{DrPO}}(z_i)\big),
  \qquad
  \mathcal{L}_{\mathrm{DrPO}}
  = \frac{1}{2K} \sum_{i=1}^{K} \big\lVert z_i - z^\star_i \big\rVert_2^2 .
  \label{eq:drpo-loss}
\end{equation}
 In our offline setting, we replace the online feature sets $\mathcal{A^+}$ and $\mathcal{A^-}$ with features of positive and negative samples from the dataset. A detailed explanation is provided in \cite{jiang2026drifting}.

\section{Experiments Details}
\label{app:experiments details}
\subsection{Toy Experiments}
\label{app:toy experiments}
\begin{figure}[!ht]
    \centering
    \includegraphics[width=1\linewidth]{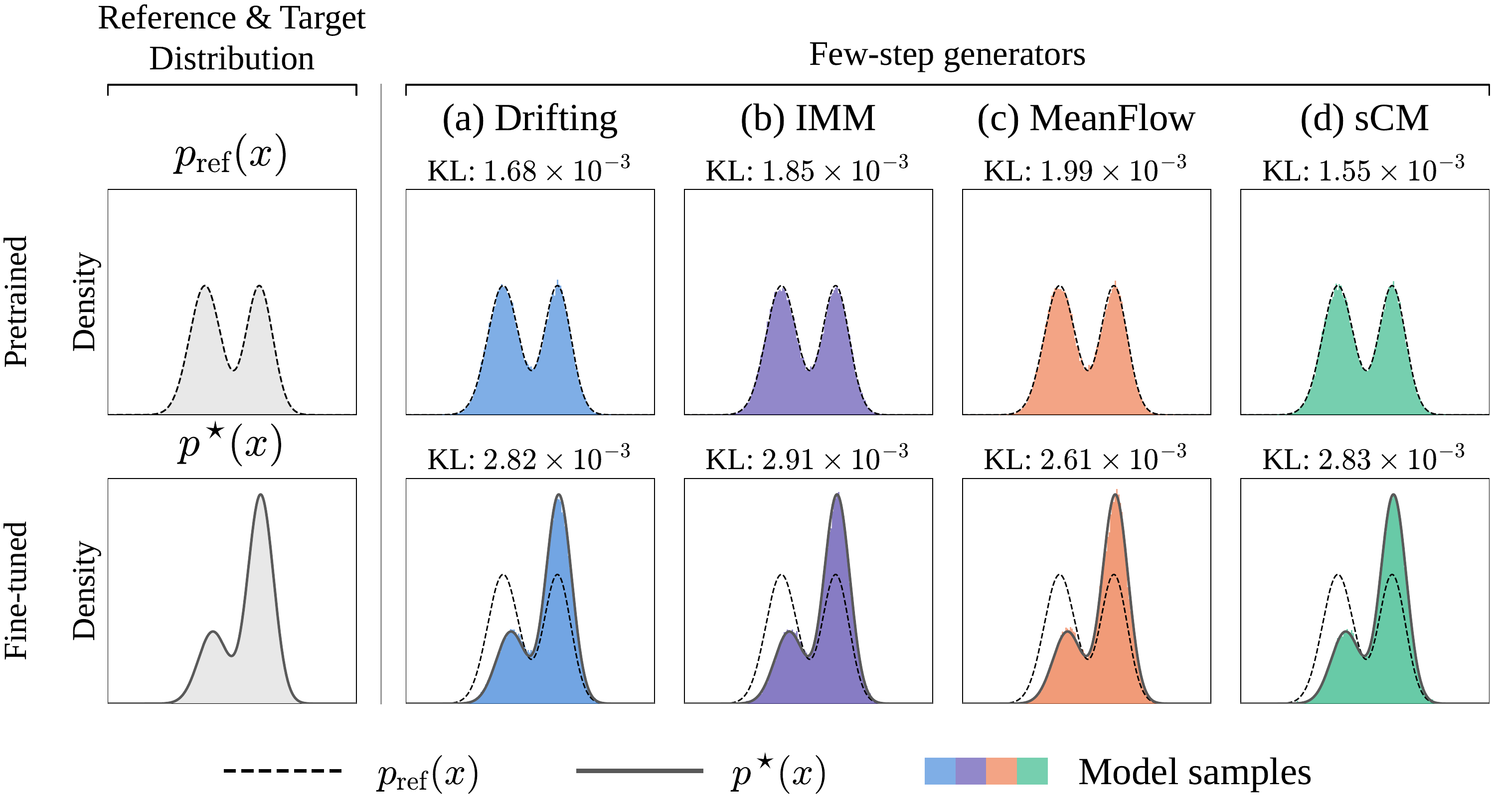}
    \captionsetup{skip=2pt}
    \caption{\textbf{FestDPO aligns diverse few-step generators with the target distribution.} Left of the solid line: the reference distribution $p_{\text{ref}}(x)$, a mixture of Gaussians, and the reward-tilted target distribution $p^\star(x) \propto p_{\text{ref}}(x)\exp(r^{\star}(x))$. Right of the solid line: distributions of (a) Drifting model, (b) IMM, (c) MeanFlow, and (d) sCM) before (top) and after (bottom) fine-tuning with FestDPO. Dashed and solid curves denote $p_{\text{ref}}(x)$ and $p^\star(x)$, respectively. We report the KL divergence (KL) between the fine-tuned and target distributions.}
    \label{fig:app toy}
    \vspace{-10pt}
\end{figure}

\paragraph{Reference, Reward, and Target distribution.}
The reference distribution $p_{\text{ref}}(x)$ is a two-component Gaussian mixture with means $(-0.625, 0.25)$, variances $(0.0625, 0.046875)$, and weights $(0.5364, 0.4636)$. The reward $r^{\star}(x)$ is a Gaussian density $\mathcal{N}(x; 0.5, 0.5)$, which peaks slightly to the right of the right mode of $p_{\text{ref}}(x)$. The target distribution is the reward-tilted distribution $p^\star(x) \propto p_{\text{ref}}(x)\exp(r^{\star}(x))$, which is itself a Gaussian mixture available in closed form. Hence, forward KL divergence is computed against the exact target.

\paragraph{Dataset.} 
For pre-training, we draw $10^5$ i.i.d.\ samples from $p(x)$ once and use this fixed dataset for all methods. For fine-tuning, we construct a fixed offline preference dataset of $10^5$ pairs, shared across all methods, by drawing two independent samples from the reference model and labeling them with the Bradley-Terry model, i.e., $y_a$ is preferred with probability $\sigma(\log r(y_a) - \log r(y_b))$

\paragraph{Pre-training.}
All methods share the same backbone: a time-conditioned residual MLP with 1-dimensional input and output, a 128-dimensional sinusoidal time embedding, SiLU activations, and four residual blocks. The hidden width is 256 (512 for sCM). The pre-training hyperparameters are summarized in \Cref{tab:gmm-pretrain}.

\begin{table}[!ht]
\centering
\small
\caption{Pre-training hyperparameters.}
\label{tab:gmm-pretrain}
\begin{tabular}{lcccc}
\toprule
& Drifting & IMM & MeanFlow & sCM  \\
\midrule
Training steps  & 300k & 100k & 300k & 100k \\
Optimizer       & AdamW & RAdam & Adam & AdamW \\
Learning rate   & $2\times10^{-4}$ & $1\times10^{-4}$ & $2\times10^{-4}$ & $3\times10^{-5}$ \\
$(\beta_1, \beta_2)$ & (0.9, 0.95) & (0.9, 0.999) & (0.9, 0.999) & (0.9, 0.99) \\
Weight decay    & 0.01 & 0 & 0 & 0.01 \\
Batch size      & 1024 & 4096 & 1024 & 16384 \\
Gradient clip   & 2.0 & -- & -- & 10.0 \\
EMA decay       & 0.999 & 0.9999 & 0.9999 & 0.999 \\
\# Parameters   & 0.53M & 0.92M & 0.92M & 3.48M \\
\bottomrule
\end{tabular}
\end{table}
\paragraph{Fine-tuning with FestDPO.}
Both log-densities $\pi_{\text{ref}}$ and $\pi_\theta$ are estimated with a Gaussian KDE (\Cref{eq : density approximation}): $\log \hat\pi_\theta$ and $\log \hat\pi_{\text{ref}}$ are computed from $N$ samples respectively. The loss is averaged over multiple KDE bandwidths, listed in \Cref{tab:gmm-finetune}. Averaging over bandwidths at multiple scales provides informative gradients commonly used in kernel-based distribution matching \citep{binkowski2018demystifying}. Under this setup, the optimum of the DPO objective with $\beta = 1$ is exactly $p^\star(x) \propto p(x)\,r(x)$. The fine-tuning hyperparameters are summarized in \Cref{tab:gmm-finetune}.

\begin{table}[!ht]
\centering
\small
\caption{Fine-tuning hyperparameters of FestDPO.}
\label{tab:gmm-finetune}
\begin{tabular}{lcccc}
\toprule
& Drifting & IMM & MeanFlow & SCM \\
\midrule
Fine-tuning steps       & 4,000 & 4,000 & 8,000 & 4,000 \\
Learning rate & $1\times10^{-4}$ & $1\times10^{-4}$ & $3\times10^{-5}$ & $1\times10^{-4}$ \\
$\beta$                 & 1.0 & 1.0 & 1.0 & 1.0 \\
KDE bandwidths          & 0.02 / 0.06 / 0.2 & 0.02 / 0.06 / 0.2 & 0.01 / 0.03 / 0.1 & 0.02 / 0.06 / 0.2 \\
\# Samples ($M = N$)    & 2048 & 2048 & 8192 & 2048 \\
Pairs per step          & 512 & 512 & 512 & 512 \\
Gradient clip           & 1.0 & 1.0 & 1.0 & 1.0 \\
\bottomrule
\end{tabular}
\end{table}

\paragraph{Sampling.}
We use 1 sampling step for Drifting and 4 steps for IMM, MeanFlow, and sCM, both during fine-tuning and at evaluation. All samples are generated with EMA weights, and each reported metric is computed from $10^5$ samples per checkpoint.

\subsection{Text-to-image generation}
\label{app:T2I implementation details}
\subsubsection{Datasets Details}
In this section, we provide a detailed explanation of the datasets.

\textbf{Pick-a-Pic v2} \citep{kirstain2023pick}: Pick-a-Pic v2 collects human preferences from users of a text-to-image web application, where each sample pairs a prompt with two generated images and a preference label. It provides 851,293 pairs across 58,960 prompts, with images generated by Stable Diffusion 2.1, Dreamlike Photoreal 2.0, and Stable Diffusion XL variants \citep{rombach2022high} using diverse classifier-free guidance scales \citep{ho2022classifier}. For evaluation, we use all 500 prompts in the test split.

\textbf{Parti-Prompts} \citep{yu2022scaling}: Parti-Prompts (P2) is a dataset of 1,632 English prompts designed to probe text-to-image models across a broad range of capabilities. Each prompt is annotated with two labels: a \emph{Category}, which specifies its broad subject domain, and a \emph{Challenge}, which identifies the aspect that makes the prompt difficult to render faithfully.  For evaluation, we construct a subset of 500 prompts by sampling an equal number of prompts from each category.

\textbf{HPSv2} \citep{wu2023human}: HPSv2 contains approximately 798K human preference choices over 434K image pairs generated by a variety of text-to-image models, together with real images. It also provides a benchmark of 3,200 prompts evenly divided into four styles: Animation, Concept-art, Painting, and Photo. For evaluation, we sample 125 prompts from each style, resulting in 500 prompts in total.

\subsubsection{Implementation details}
In this section, we describe the experimental setup and implementation details of FestDPO. We apply LoRA \citep{hu2021lora} fine-tuning to the base few-step generator (SDXL-Turbo and SDXL-DMD2), using a rank of 16 and a scaling factor of 1 across all experiments, and conduct all main experiments with single-step generation (NFE $=1$). We use the AdamW optimizer \citep{kingma2014adam} with a prompt batch size of 512. The learning rate is initialized at $3 \times 10^{-4}$ and decayed to $1 \times 10^{-6}$ following a cosine schedule over the first 300 steps, followed by a constant learning rate of $1 \times 10^{-6}$. The regularization coefficient $\beta$ is fixed to 0.5 throughout all experiments. For each prompt, FestDPO constructs the KDE from $N=12$ generated samples. Since the density is estimated in the feature space of a pretrained encoder (Latent MAE by default), we find a simple Gaussian kernel to be sufficient, and the kernel bandwidth is tuned separately for each base model. All experiments are conducted on 8 NVIDIA RTX 3090 GPUs. Further details on the hyperparameter settings  can be found in \Cref{tab:hyperparameters}.
\begin{table}[!ht]
\centering
\caption{Hyperparameters for FestDPO.}
\label{tab:hyperparameters}
\resizebox{0.6\linewidth}{!}{
\begin{tabular}{lcc}
\toprule
& SDXL-Turbo & SDXL-DMD2 \\
\midrule
KDE bandwidth $h$ & 0.15, 0.2 & 0.2, 0.3 \\
\midrule
$\beta$                  & \multicolumn{2}{c}{0.5} \\
\# of KDE samples $N$    & \multicolumn{2}{c}{12} \\
LoRA rank / scale        & \multicolumn{2}{c}{16 / 1} \\
Optimizer                & \multicolumn{2}{c}{AdamW ($\beta_1=0.9, \beta_2=0.999$)} \\
Weight decay             & \multicolumn{2}{c}{0.01} \\
Initial learning rate    & \multicolumn{2}{c}{$3 \times 10^{-4}$} \\
Final learning rate      & \multicolumn{2}{c}{$1 \times 10^{-6}$} \\
Gradient clip norm       & \multicolumn{2}{c}{1.0} \\
Max steps                & \multicolumn{2}{c}{1,000} \\
Prompt batch size        & \multicolumn{2}{c}{512} \\
\bottomrule
\end{tabular}
}
\end{table}

\textbf{Baseline implementation details.}
Building on the official implementations of PSO and DrPO, we adapt both methods to our offline setting and tune their key hyperparameters as described below.
\begin{itemize}
    \item \textbf{PSO} \citep{miao2025tuning}: Pairwise Sample Optimization fine-tunes timestep-distilled diffusion models by increasing the relative likelihood of preferred images over reference images, where the likelihood is approximated by the denoising loss along the diffusion trajectory. Following \citet{miao2025tuning}, we tune the learning rate over $\{10^{-5}, 10^{-4}\}$ and the regularization coefficient $\beta$ over $\{1, 10, 50, 100\}$, and select $LR=10^{-4}$ and $\beta = 10$.
    \item \textbf{DrPO} \citep{jiang2026drifting}: DrPO aligns few-step generators by constructing a drift field from preference pairs, which attracts generated samples toward preferred samples and repels them from dispreferred ones, with a regularization term weighted by $\lambda$ that keeps the model close to the reference. We tune the learning rate over $\{10^{-5}, 10^{-4}\}$ and $\lambda$ over $\{0.1, 0.2, 0.3, 0.5\}$, and select $LR=10^{-4}$ and $\lambda = 0.1$. For a fair comparison, we use the same feature encoder (Latent MAE) and kernel bandwidth as FestDPO.
\end{itemize}

\subsection{Protein backbone generation}
\label{app:protein implementation details}
\paragraph{Base model and optimization targets.}
We use the RMF-S variant of Riemannian MeanFlow \citep{woo2026riemannian} as the base generator and frozen reference model. We fine-tune separate generators for two optimization targets: SS-match, adapted from \citet{su2026iterative}, and self-consistency root mean square deviation (scRMSD). The SS-match reward favors a higher fraction of residues in $\beta$-strands. We assign secondary-structure labels using DSSP \citep{kabsch1983dictionary} and group them into helices, strands, and coils. For scRMSD, we use ProteinMPNN \citep{dauparas2022robust} to design an amino-acid sequence for each generated backbone and ESMFold \citep{lin2023evolutionary} to predict its structure. We then compute the C$\alpha$ RMSD between the generated and predicted backbones after structural alignment, favoring lower values as a proxy for structural designability.

\paragraph{Model checkpoints.}
We use the RMF-S checkpoint and its accompanying config from the released weights archive\footnote{\url{https://zenodo.org/records/18582218}}, together with the official Protein-RMF implementation\footnote{\url{https://github.com/dongyeop3813/Protein-RMF}}. For both inverse folding and feature extraction, we use the C$\alpha$-only ProteinMPNN checkpoint from the official repository\footnote{\url{https://github.com/dauparas/ProteinMPNN}}. This checkpoint uses 48 neighbors and was trained with coordinate noise of $0.20\text{\AA}$. The ProteinMPNN encoder remains frozen during FestDPO training. For refolding, we use the \texttt{facebook/esmfold\_v1} checkpoint from Hugging Face\footnote{\url{https://huggingface.co/facebook/esmfold_v1}}.

\subsubsection{Preference Dataset Construction}
For each optimization target, we construct a separate offline preference dataset from a pool of $8{,}192$ backbones of length $128$ generated by the frozen RMF-S reference model. Pairs are formed only within the corresponding pool. The preference datasets remain fixed during fine-tuning, which requires no additional reward evaluations or preference annotations.
\paragraph{SS-match.}
We generate the pool using single-step sampling and score each backbone using the ceiling-normalized P-SEA $\beta$-strand fraction. For backbones of length 128, one additional strand residue changes the score by 1/(128-4)$\approx$0.0081. We require a preference gap of at least 2/(128-4)$\approx$0.0161, corresponding to two residues, to reduce sensitivity to single-residue differences at secondary-structure boundaries. The reference pool has a mean strand fraction of 0.152 with a standard deviation of 0.107.
\paragraph{scRMSD.}
We generate a separate pool using ten-step sampling to include backbones with better self-consistency than those obtained through single-step generation. Each backbone is scored by its negative scRMSD, so that higher scores indicate stronger self-consistency. The pool has a median scRMSD of 5.30\text{\AA}, and 7.4\% of its backbones satisfy the designability criterion of scRMSD below $2.0,\text{\AA}$. We require an scRMSD gap of at least 0.1\text{\AA} between paired backbones.
\paragraph{Pair construction.}
We greedily form pairs satisfying the target-specific gap threshold and allow each backbone to appear in at most four pairs, limiting repeated use of individual samples. This procedure yields 8000 pairs per target. Within each pair, the backbone with the higher SS-match score or lower scRMSD is labeled preferred.

\subsubsection{Implementation details}
We fine-tune all $16.35$ million parameters of RMF-S without adapters. Fine-tuning and evaluation use single-step generation ($K=1$, NFE $=2$) with a backbone length of $128$ residues. We use AdamW \citep{loshchilov2017decoupled} with a constant learning rate of $10^{-6}$ for $1{,}000$ updates, using $64$ preference pairs per update. At each update, FestDPO constructs density estimates from $64$ samples drawn from the trainable generator and $64$ samples drawn from the frozen reference generator. Both sample sets are regenerated at every update. We apply a Gaussian kernel in the $128$-dimensional feature space of the frozen ProteinMPNN encoder. \Cref{tab:hyperparameters-protein} lists the hyperparameters and candidate values considered.  Experiments use a single NVIDIA RTX 5090 GPU and take approximately 2 hours.

\begin{table}[!ht]
\centering
\small
\caption{Hyperparameters and candidate values for FestDPO on protein backbone generation.}
\label{tab:hyperparameters-protein}
\setlength{\tabcolsep}{5pt}
\begin{tabular}{lcc}
\toprule
Hyperparameter & SS-match & scRMSD \\
\midrule
KDE bandwidth $\sigma$ & $ 0.5 $ & $0.05$ \\
DPO coefficient $\beta$ & $ 2000 $ & $20$ \\
\midrule
Base model & \multicolumn{2}{c}{RMF-S} \\
KDE samples per update & \multicolumn{2}{c}{$64$ trainable / $64$ reference} \\
Optimizer & \multicolumn{2}{c}{AdamW} \\
AdamW $(\beta_1,\beta_2)$ & \multicolumn{2}{c}{$(0.9,0.999)$} \\
Weight decay & \multicolumn{2}{c}{$10^{-12}$} \\
Learning rate & \multicolumn{2}{c}{$10^{-6}$} \\
Gradient clipping norm & \multicolumn{2}{c}{$1000$} \\
Preference pairs per update & \multicolumn{2}{c}{$64$} \\
Training updates & \multicolumn{2}{c}{$1{,}000$} \\
Backbone length & \multicolumn{2}{c}{$128$} \\
\bottomrule
\end{tabular}
\end{table}

\paragraph{DrPO baseline.}
We use a Laplacian kernel for DrPO and consider kernel radii of $(0.02,0.05,0.2)$. We tune the reference loss weight $\lambda$ over the range $[0.01, 0.02, 0.05, 0.1, 0.2, 0.5, 1, 2]$, while fixing the anchor weight at $1.0$. Each update uses $64$ samples from the trainable generator and $64$ reference samples. The remaining training settings, including the learning rate, number of updates, preference batch size, and random seeds, are identical to those used for FestDPO.

\clearpage
\section{Additional Experimental Results}
We evaluate all methods on 500 prompts from each of Pick-a-Pic v2, Parti-Prompts, and HPSv2. For each prompt, we generate images with four different random seeds. In all experiments, we report the mean and standard deviation of PickScore \citep{kirstain2023pick} and Aesthetic score \citep{schuhmann2023laion} over each prompt set.
\label{app:additional experiments}

\subsection{Main Results}
\begin{table}[!ht]
\centering
\small
\setlength{\tabcolsep}{5pt}
\caption{Main results with SDXL-Turbo on Pick-a-Pic v2, Parti-Prompts, HPSv2.}
\label{tab:sdxl-turbo-overall}
\begin{tabular}{lcccccc}
\toprule
& \multicolumn{2}{c}{Pick-a-Pic v2} & \multicolumn{2}{c}{Parti-Prompts} & \multicolumn{2}{c}{HPSv2} \\
\cmidrule(lr){2-3} \cmidrule(lr){4-5} \cmidrule(lr){6-7}
Method & PickScore $\uparrow$ & Aesthetic $\uparrow$ & PickScore $\uparrow$ & Aesthetic $\uparrow$ & PickScore $\uparrow$ & Aesthetic $\uparrow$ \\
\midrule
SDXL-Turbo & 22.37 $\pm$ 0.06 & 6.03 $\pm$ 0.03 & 22.78 $\pm$ 0.04 & 5.69 $\pm$ 0.02 & 22.83 $\pm$ 0.05 & 6.09 $\pm$ 0.02 \\
DrPO       & 22.38 $\pm$ 0.06 & 6.03 $\pm$ 0.03 & 22.78 $\pm$ 0.04 & 5.69 $\pm$ 0.02 & 22.85 $\pm$ 0.05 & 6.09 $\pm$ 0.02 \\
PSO        & 22.46 $\pm$ 0.06 & 6.06 $\pm$ 0.02 & 22.86 $\pm$ 0.04 &  \cellcolor{festblue}\textbf{5.72 $\pm$ 0.02} & \cellcolor{festblue}\textbf{23.22 $\pm$ 0.05} & 6.05 $\pm$ 0.02 \\
FestDPO    &\cellcolor{festblue} \textbf{22.60 $\pm$ 0.07} & \cellcolor{festblue}\textbf{6.07 $\pm$ 0.03} & \cellcolor{festblue}\textbf{22.92 $\pm$ 0.05} & \cellcolor{festblue}\textbf{5.72 $\pm$ 0.02} & 23.18 $\pm$ 0.05 &\cellcolor{festblue} \textbf{6.14 $\pm$ 0.02} \\
\bottomrule
\end{tabular}
\end{table}

\begin{table}[!ht]
\centering
\small
\setlength{\tabcolsep}{5pt}
\caption{Win-rate comparison with SDXL-Turbo on Pick-a-Pic v2, Parti-Prompts, HPSv2.}
\label{tab:sdxl-turbo-win-rate}
\begin{tabular}{lcccccc}
\toprule
& \multicolumn{2}{c}{Pick-a-Pic v2} & \multicolumn{2}{c}{Parti-Prompts} & \multicolumn{2}{c}{HPSv2} \\
\cmidrule(lr){2-3} \cmidrule(lr){4-5} \cmidrule(lr){6-7}
Method & PickScore $\uparrow$ & Aesthetic $\uparrow$ & PickScore $\uparrow$ & Aesthetic $\uparrow$ & PickScore $\uparrow$ & Aesthetic $\uparrow$ \\
\midrule
DrPO    & 51.65 $\pm$ 1.21 & 49.40 $\pm$ 1.27 & 50.20 $\pm$ 1.00 & 53.65 $\pm$ 1.02 & 54.50 $\pm$ 1.13 & 51.30 $\pm$ 1.09 \\
PSO     & 65.60 $\pm$ 1.29 & 56.20 $\pm$ 1.35 & 64.70 $\pm$ 1.09 & 57.05 $\pm$ 1.08 & 77.70 $\pm$ 1.03 & 44.35 $\pm$ 1.33 \\
FestDPO & \cellcolor{festblue}\textbf{73.55 $\pm$ 1.24} & \cellcolor{festblue} \textbf{59.45 $\pm$ 1.27} & \cellcolor{festblue} \textbf{67.65 $\pm$ 1.08} &  \cellcolor{festblue}\textbf{63.05 $\pm$ 1.04} &  \cellcolor{festblue}\textbf{82.25 $\pm$ 0.96} & \cellcolor{festblue} \textbf{61.00 $\pm$ 1.18} \\
\bottomrule
\end{tabular}
\end{table}

\begin{table}[!ht]
\centering
\small
\setlength{\tabcolsep}{5pt}
\caption{Main results with SDXL-DMD2 on Pick-a-Pic v2, Parti-Prompts, HPSv2.}
\label{tab:dmd2-overall}
\begin{tabular}{lcccccc}
\toprule
& \multicolumn{2}{c}{Pick-a-Pic v2} & \multicolumn{2}{c}{Parti-Prompts} & \multicolumn{2}{c}{HPSv2} \\
\cmidrule(lr){2-3} \cmidrule(lr){4-5} \cmidrule(lr){6-7}
Method & PickScore $\uparrow$ & Aesthetic $\uparrow$ & PickScore $\uparrow$ & Aesthetic $\uparrow$ & PickScore $\uparrow$ & Aesthetic $\uparrow$ \\
\midrule
SDXL-DMD2 & 21.06 $\pm$ 1.46 & 5.37 $\pm$ 0.68  & 21.65 $\pm$ 1.17 &5.25 $\pm$ 0.57& 21.45 $\pm$ 1.28 & 5.58 $\pm$ 0.67 \\
DrPO      & 21.17 $\pm$ 1.47 & 5.43 $\pm$ 0.70 & 21.73 $\pm$ 1.19  &  5.31 $\pm$ 0.57 & 21.61 $\pm$ 1.31 & 5.63 $\pm$ 0.68 \\
PSO   & 21.58 $\pm$ 1.54 & 5.66 $\pm$ 0.70 & 22.06 $\pm$ 1.35 & 5.51 $\pm$ 0.64 & 22.05 $\pm$ 1.38 & 5.90 $\pm$ 0.62 \\
FestDPO & \cellcolor{festblue}\textbf{21.98 $\pm$ 1.50} &\cellcolor{festblue}\textbf{5.87 $\pm$ 0.57}  &\cellcolor{festblue}\textbf{22.43 $\pm$ 1.16} & \cellcolor{festblue} \textbf{5.65 $\pm$ 0.52} & \cellcolor{festblue} \textbf{22.42 $\pm$ 1.32} & \cellcolor{festblue} \textbf{5.99 $\pm$ 0.59} \\

\bottomrule
\end{tabular}
\end{table}

\begin{table}[!ht]
\centering
\small
\setlength{\tabcolsep}{5pt}
\caption{Win-rate comparison with SDXL-DMD2 on Pick-a-Pic v2, Parti-Prompts, HPSv2.}
\label{tab:dmd2-win-rate}
\begin{tabular}{lcccccc}
\toprule
& \multicolumn{2}{c}{Pick-a-Pic v2} & \multicolumn{2}{c}{Parti-Prompts} & \multicolumn{2}{c}{HPSv2} \\
\cmidrule(lr){2-3} \cmidrule(lr){4-5} \cmidrule(lr){6-7}
Method & PickScore $\uparrow$ & Aesthetic $\uparrow$ & PickScore $\uparrow$ & Aesthetic $\uparrow$ & PickScore $\uparrow$ & Aesthetic $\uparrow$ \\
\midrule
DrPO       & 67.00 $\pm$ 1.67 & 61.55 $\pm$ 1.98 & 63.15 $\pm$ 2.29 & 63.00 $\pm$ 1.99 & 72.00 $\pm$ 1.21  & 63.35 $\pm$ 1.19 \\
PSO        &  77.00 $\pm$ 0.42 & 76.15 $\pm$ 0.43 & 73.85 $\pm$ 0.44 & 73.70 $\pm$ 0.44  & 78.90 $\pm$ 0.41 & 78.60 $\pm$ 0.41  \\
FestDPO    & \cellcolor{festblue}\textbf{88.30 $\pm$ 0.72} & \cellcolor{festblue}\textbf{87.25 $\pm$ 0.75} & \cellcolor{festblue}\textbf{86.55 $\pm$ 0.76} & \cellcolor{festblue} \textbf{84.60 $\pm$ 0.81} & \cellcolor{festblue} \textbf{90.90 $\pm$ 0.64} &\cellcolor{festblue} \textbf{85.05 $\pm$ 0.80} \\

\bottomrule
\end{tabular}
\end{table}

\subsection{Number of KDE samples}
Since FestDPO employs nonparametric density estimation, we first study the effect of the number of samples $N$ used to construct the KDE. Beyond HPSv2 reported in \Cref{tab:kde sample}, we additionally evaluate on Pick-a-Pic v2 and Parti-Prompts. As shown in \Cref{tab:abl-batch}, increasing $N$ from 8 to 24 consistently improves PickScore across all three benchmarks, while Aesthetic scores remain approximately stable. These results indicate that FestDPO is robust to the choice of $N$, with even a small number of samples yielding substantial gains over the pretrained model.

\begin{table}[!ht]
\centering
\small
\setlength{\tabcolsep}{5pt}
\caption{Sensitivity analysis on the number of KDE samples with SDXL-Turbo.}
\label{tab:abl-batch}
\begin{tabular}{lcccccc}
\toprule
& \multicolumn{2}{c}{Pick-a-Pic v2} & \multicolumn{2}{c}{Parti-Prompts} & \multicolumn{2}{c}{HPSv2} \\
\cmidrule(lr){2-3} \cmidrule(lr){4-5} \cmidrule(lr){6-7}
\# of KDE samples & PickScore $\uparrow$ & Aesthetic $\uparrow$ & PickScore $\uparrow$ & Aesthetic $\uparrow$ & PickScore $\uparrow$ & Aesthetic $\uparrow$ \\
\midrule
SDXL-Turbo  & 22.37 $\pm$ 1.46 & 6.03 $\pm$ 0.61 &22.78 $\pm$ 1.20 & 5.69 $\pm$ 0.57 &22.83 $\pm$ 1.29 & 6.09 $\pm$ 0.69 \\
$N=8$  &22.53 $\pm$ 1.51 & 6.06 $\pm$ 0.62 &22.91 $\pm$ 1.24 & 5.74 $\pm$ 0.58 &  23.06 $\pm$ 1.32 & 6.13 $\pm$ 0.68  \\
$N=12$ & 22.60 $\pm$ 1.52& 6.07 $\pm$ 0.62 &22.92 $\pm$ 1.25 & 5.73 $\pm$ 0.59   &23.13 $\pm$ 1.35 & 6.11 $\pm$ 0.66 \\
$N=24$ & 22.61 $\pm$ 1.50& 6.02 $\pm$ 0.60 &22.94 $\pm$ 1.24 & 5.72 $\pm$ 0.57   &23.14 $\pm$ 1.37 & 6.07 $\pm$ 0.63 \\
\bottomrule
\end{tabular}
\end{table}
\subsection{Feature encoder for KDE}
We next conduct an ablation on the feature encoder used for kernel density estimation. \Cref{tab:abl-feature} compares three pretrained encoders: Latent MAE, CLIP, and DINOv2. FestDPO achieves comparable performance across the three pretrained encoders, demonstrating that its effectiveness is robust to the choice of pretrained encoder. We adopt Latent MAE as the default encoder, as it attains consistently strong results on both PickScore and Aesthetic score, and apply the same encoder to DrPO for a fair comparison.
\begin{table}[!ht]
\centering
\small
\setlength{\tabcolsep}{5pt}
\caption{Ablation on the feature encoder with SDXL-Turbo.}
\label{tab:abl-feature}
\begin{tabular}{lcccccc}
\toprule
& \multicolumn{2}{c}{Pick-a-Pic v2} & \multicolumn{2}{c}{Parti-Prompts} & \multicolumn{2}{c}{HPSv2} \\
\cmidrule(lr){2-3} \cmidrule(lr){4-5} \cmidrule(lr){6-7}
Feature encoder& PickScore $\uparrow$ & Aesthetic $\uparrow$ & PickScore $\uparrow$ & Aesthetic $\uparrow$ & PickScore $\uparrow$ & Aesthetic $\uparrow$ \\
\midrule
SDXL-Turbo  & 22.37 $\pm$ 1.46 & 6.03 $\pm$ 0.61 &22.78 $\pm$ 1.20 & 5.69 $\pm$ 0.57 &22.83 $\pm$ 1.29 & 6.09 $\pm$ 0.69 \\
Latent MAE  & 22.60 $\pm$ 1.52 & 6.07 $\pm$ 0.62 &22.92 $\pm$ 1.25 & 5.73 $\pm$ 0.59 &23.13 $\pm$ 1.35 & 6.11 $\pm$ 0.66 \\
CLIP       &  22.66 $\pm$ 1.48 & 5.99 $\pm$ 0.57 &22.91 $\pm$ 1.22 & 5.72 $\pm$ 0.56 & 23.17 $\pm$ 1.35 & 6.08 $\pm$ 0.64  \\
DINOv2       & 22.54 $\pm$ 1.49 & 6.02 $\pm$ 0.61 &22.92 $\pm$ 1.22 & 5.71 $\pm$ 0.59 & 23.07 $\pm$ 1.34 & 6.08 $\pm$ 0.65 \\
\bottomrule
\end{tabular}
\end{table}

\paragraph{Number of function evaluations (NFE).}
While our main experiments are reported with 1-step generation, we further examine whether FestDPO remains effective in the few-step regime. As shown in \Cref{tab:abl-nfe}, FestDPO with two function evaluations (NFE $=2$) achieves performance comparable to the 1-step setting, with consistent improvements in PickScore across all three benchmarks. This suggests that FestDPO benefits from the improved sample quality afforded by additional function evaluations.
\begin{table}[!ht]
\centering
\small
\setlength{\tabcolsep}{5pt}
\caption{Performance of FestDPO with SDXL-Turbo under different NFE.}
\label{tab:abl-nfe}
\begin{tabular}{lcccccc}
\toprule
& \multicolumn{2}{c}{Pick-a-Pic v2} & \multicolumn{2}{c}{Parti-Prompts} & \multicolumn{2}{c}{HPSv2} \\
\cmidrule(lr){2-3} \cmidrule(lr){4-5} \cmidrule(lr){6-7}
NFE & PickScore $\uparrow$ & Aesthetic $\uparrow$ & PickScore $\uparrow$ & Aesthetic $\uparrow$ & PickScore $\uparrow$ & Aesthetic $\uparrow$ \\
\midrule
1  & 22.60 $\pm$ 1.52& 6.07 $\pm$ 0.62 &22.92 $\pm$ 1.25 & 5.73 $\pm$ 0.59 & 23.13 $\pm$ 1.35 & 6.11 $\pm$ 0.66 \\
2 & 22.63 $\pm$ 1.48 & 5.98 $\pm$ 0.59 &23.07 $\pm$ 1.22 & 5.67 $\pm$ 0.54 & 23.14 $\pm$ 1.33 & 6.06 $\pm$ 0.67 \\
\bottomrule
\end{tabular}
\end{table}

\clearpage
\section{Qualitative Results}
\label{app: Qualitative_results}
In this section, we present qualitative comparisons of FestDPO with PSO and DrPO for text-to-image generation across three prompts sets: Pick-a-Pic v2, PartiPrompts, and HPS v2. All methods use SDXL-Turbo as the base model.

\subsection{Pick-a-Pic V2}

\begin{figure}[H]
    \centering
    \includegraphics[
        width=\linewidth,
        height=0.75\textheight,
        keepaspectratio
    ]{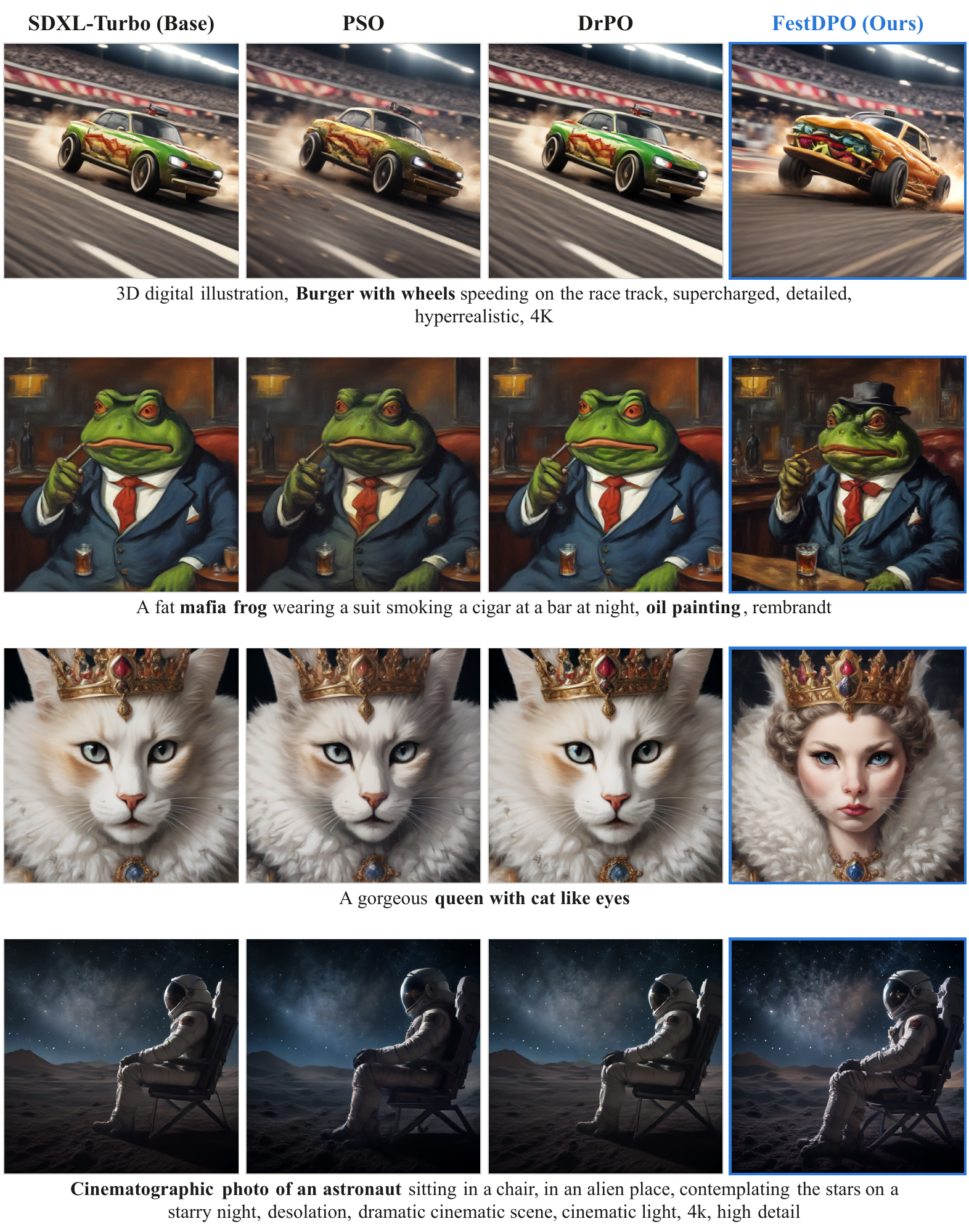}
    \caption{Qualitative comparison of FestDPO with PSO and DrPO on Pick-a-Pic V2.}
    \label{fig:qualitative_pickapic}
\end{figure}
\clearpage

\subsection{Parti-Prompts}
\begin{figure}[H]
    \centering
    \includegraphics[
        width=\linewidth,
        height=0.75\textheight,
        keepaspectratio
    ]{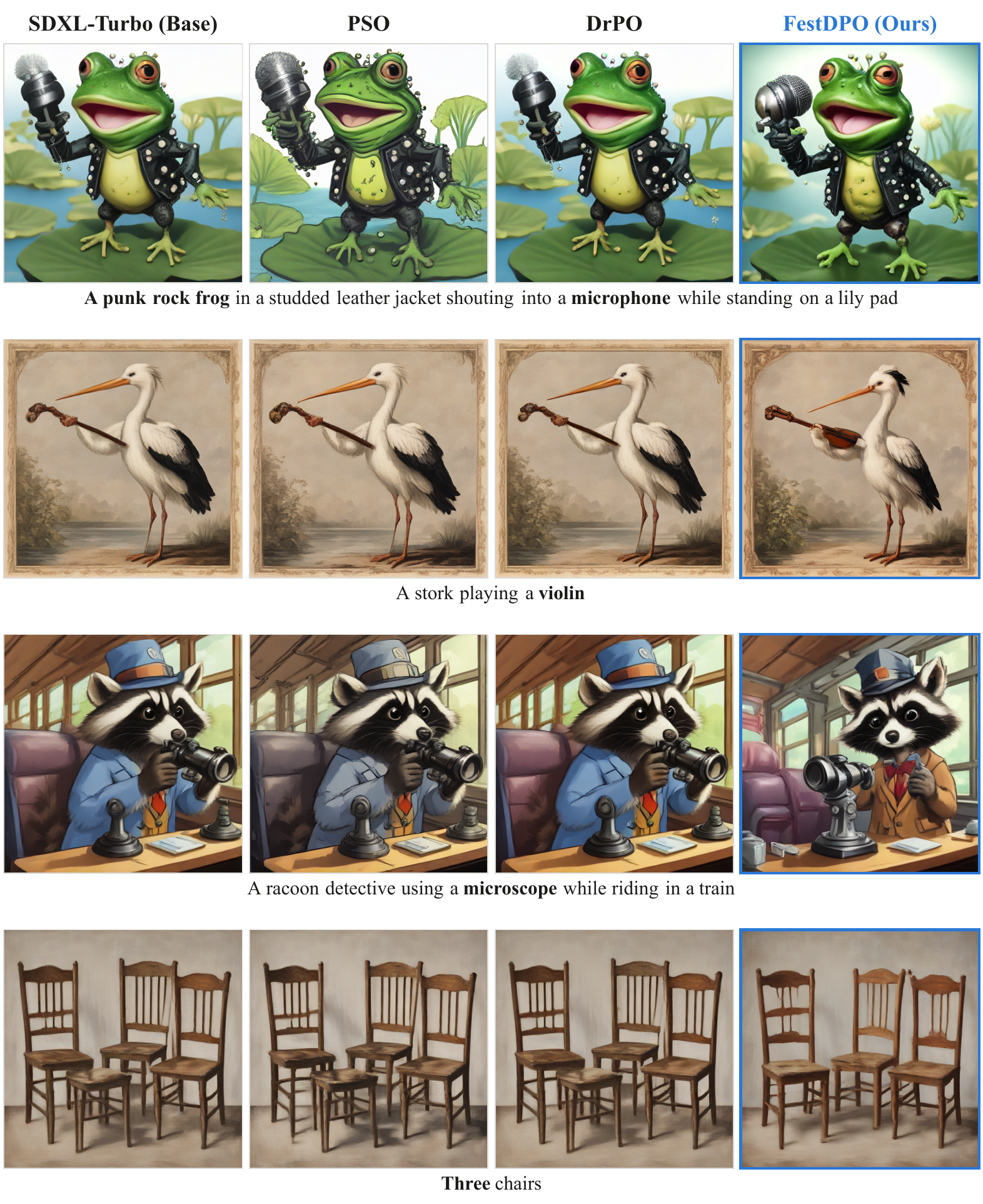}
    \caption{Qualitative comparison of FestDPO with PSO and DrPO on Parti-Prompts.}
    \label{fig:qualitative_pickapic}
\end{figure}
\clearpage

\subsection{HPSv2}
\begin{figure}[H]
    \centering
    \includegraphics[
        width=\linewidth,
        height=0.75\textheight,
        keepaspectratio
    ]{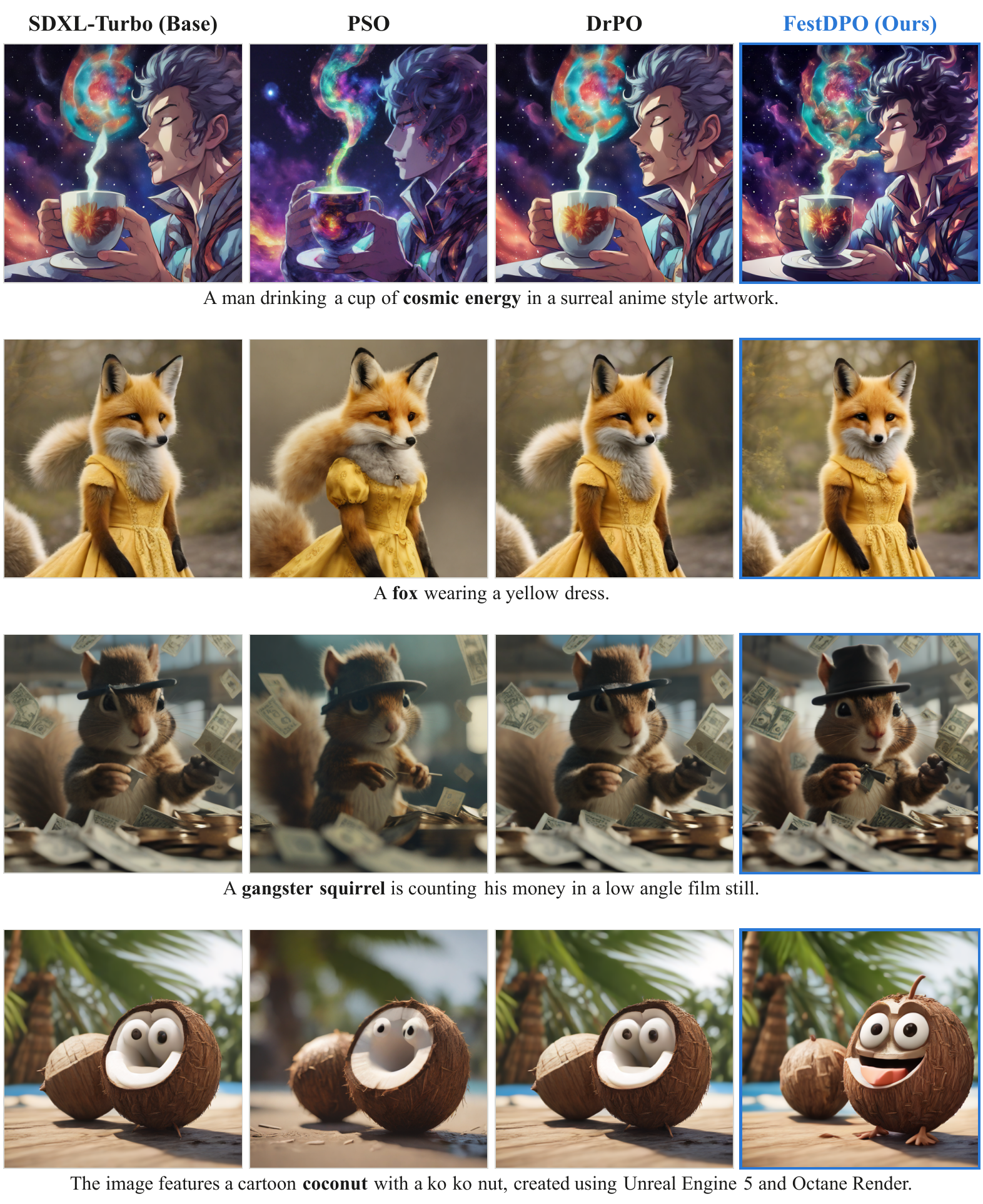}
    \caption{Qualitative comparison of FestDPO with PSO and DrPO on HPSv2.}
    \label{fig:qualitative_pickapic}
\end{figure}

\clearpage

\section{Human Evaluation}\label{app:human evaluation}

\textbf{Evaluation protocol.}
We conduct the human evaluation using the annotation interface shown in \Cref{fig:human-eval-ui}.
The evaluation includes 80 prompts randomly drawn from the three prompt sets used in our experiments: 29 from Pick-a-Pic v2, 20 from Parti-Prompts, and 31 from HPSv2.
For each prompt, annotators are presented with four images generated by SDXL-Turbo, PSO, DrPO, and FestDPO and rate each image on a 1--5 scale for prompt alignment and aesthetics.
To prevent annotators from identifying which method generates each image, all images are anonymized.
For each annotator, we independently randomize both the prompt order and the arrangement of the four images.

\textbf{Participants.}
We recruited 15 annotators from collaborators, lab members, and friends of the authors for the human evaluation.
In total, we collected 9,600 ratings (15 annotators $\times$ 320 images $\times$ 2 evaluation criteria).
The median completion time per annotator is about 36 minutes.

\textbf{Results and reliability.}
As shown in \Cref{tab:human-eval-full}, FestDPO achieves the highest mean ratings for both prompt alignment and aesthetics across all three prompt sets. The improvements over all baselines are statistically significant ($p < 0.001$, Wilcoxon signed-rank test \citep{wilcoxon1945individual}). The pairwise comparison in \Cref{fig:human-eval-winrate} further shows consistently higher win rates for FestDPO against all baselines on both criteria. Specifically, FestDPO achieves win rates of 76--84\% for prompt alignment and 72--81\% for aesthetics against the three baselines.
Moreover, all 15 annotators individually assign FestDPO the highest mean rating on both criteria.
The mean pairwise Spearman correlation between annotators is 0.35 for prompt alignment and 0.32 for aesthetics.
Finally, the mean ratings differ by at most 0.06 across the four display positions, suggesting that the results are not substantially affected by image placement.

\begin{figure}[H]
    \centering
    \includegraphics[
        width=1.\linewidth,
        height=0.75\textheight,
        keepaspectratio
    ]{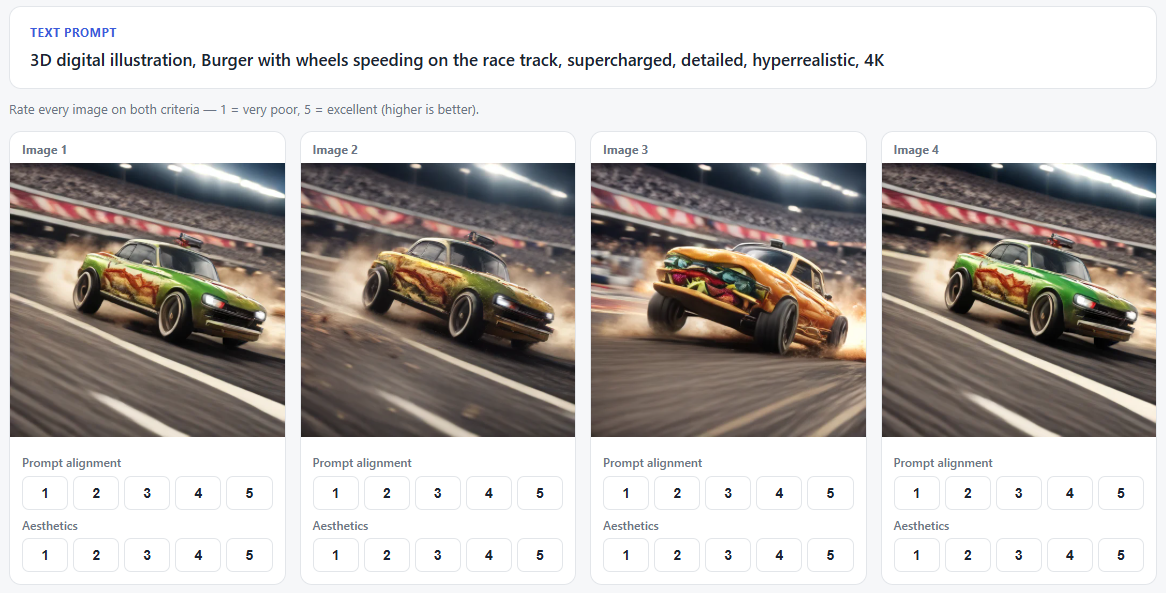}
    \caption{Screenshot of the annotation interface used for human evaluation. For each text prompt, four images generated by different methods are displayed side by side in randomized order. Annotators rate each image on a 1--5 scale for prompt alignment and aesthetics.}
    \label{fig:human-eval-ui}
\end{figure}

\begin{table}[!ht]
\centering
\small
\setlength{\tabcolsep}{5pt}
\caption{Human evaluation results on Pick-a-Pic v2, Parti-Prompts, HPSv2.}
\label{tab:human-eval-full}
\begin{tabular}{lcccccccc}
\toprule
 & \multicolumn{2}{c}{Overall} & \multicolumn{2}{c}{Pick-a-Pic v2} & \multicolumn{2}{c}{Parti-Prompts} & \multicolumn{2}{c}{HPSv2} \\
\cmidrule(lr){2-3} \cmidrule(lr){4-5} \cmidrule(lr){6-7} \cmidrule(lr){8-9}
Method & Align. $\uparrow$ & Aes. $\uparrow$ & Align. $\uparrow$ & Aes. $\uparrow$ & Align. $\uparrow$ & Aes. $\uparrow$ & Align. $\uparrow$ & Aes. $\uparrow$ \\
\midrule
SDXL-Turbo & 3.683 & 3.149 & 3.756 & 3.248 & 3.497 & 3.080 & 3.735 & 3.101 \\
PSO & 3.767 & 3.286 & 3.731 & 3.156 & 3.607 & 3.103 & 3.903 & 3.525 \\
DrPO & 3.696 & 3.158 & 3.759 & 3.234 & 3.520 & 3.057 & 3.751 & 3.151 \\
FestDPO (ours) & \textbf{4.199} & \textbf{3.860} & \textbf{4.329} & \textbf{3.936} & \textbf{4.137} & \textbf{3.793} & \textbf{4.118} & \textbf{3.832} \\
\bottomrule
\end{tabular}
\end{table}

\begin{figure}[H]
    \centering
    \includegraphics[
        width=.7\linewidth,
        keepaspectratio
    ]{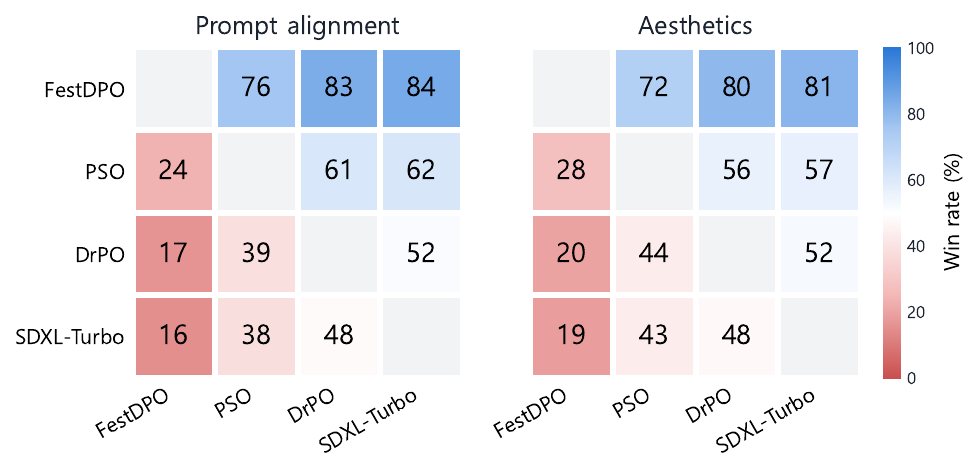}
    \caption{Pairwise win rates (\%) in the human evaluation for prompt alignment and aesthetics.
    Each entry indicates the percentage of comparisons in which
    the method on the y-axis is preferred over the method on the x-axis.}
    \label{fig:human-eval-winrate}
\end{figure}

\end{document}

%% file: math_commands.tex
\usepackage{amsmath,amsfonts,bm}

\def\eqref#1{equation~\ref{#1}}

\def\1{\bm{1}}

\DeclareMathAlphabet{\mathsfit}{\encodingdefault}{\sfdefault}{m}{sl}
\SetMathAlphabet{\mathsfit}{bold}{\encodingdefault}{\sfdefault}{bx}{n}

